\documentclass[sigconf,nonacm]{acmart}

\setcopyright{none}
\renewcommand\shortauthors{}

\AtBeginDocument{%
  }

\usepackage{bm}
\usepackage{algorithm}
\usepackage{algorithmic}
\usepackage{tcolorbox}

\newcommand{\x}{\mathbf{x}}

\begin{document}

\title[Hierarchical MoE with Two-Stage Optimization]%
{Hierarchical Mixture-of-Experts with Two-Stage Optimization}

\author{Gleb Molodtsov}
\authornote{These authors contributed equally to this work.}
\affiliation{%
  \institution{BRAIn Lab}
  \city{Moscow}
  \country{Russia}
}

\author{Alexander Miasnikov}
\authornotemark[1]
\affiliation{%
  \institution{BRAIn Lab}
  \city{Moscow}
  \country{Russia}
}

\author{Aleksandr Beznosikov}
\affiliation{%
  \institution{BRAIn Lab}
  \city{Moscow}
  \country{Russia}
}

\settopmatter{authorsperrow=3}
\renewcommand{\shortauthors}{Molodtsov et al. 2026}

\begin{abstract}
Sparse Mixture-of-Experts (MoE) models scale capacity by routing each token to a small subset of experts. However, their routers exhibit a fundamental trade-off: strong load balancing can suppress expert specialization, while aggressive diversity often causes routing collapse.
We propose $\textbf{Hi-MoE}$, a grouped MoE framework that decomposes routing control into two coupled levels: (i) $\textit{inter-group}$ balancing that enforces fair traffic across expert groups, and (ii) $\textit{intra-group}$ specialization that promotes complementary expert behaviors while preventing within-group collapse.
Our analysis provides a principled explanation of how our hierarchical objectives reshape the router, thereby promoting stable specialization and mitigating collapse.
We observe consistent improvements over recent sparse-routing and grouped-MoE baselines across NLP and vision benchmarks, and confirm robustness via scaling studies (model size, expert count) and targeted ablations. 
In large-scale pre-training on 58B tokens, \texttt{Hi-MoE}-7B achieves a $5.6\%$ perplexity reduction and a $40\%$ improvement in expert balance over \texttt{OLMoE}-7B across diverse evaluation domains. 
\end{abstract}

% %%
% %% The code below is generated by the tool at http://dl.acm.org/ccs.cfm.
% %% Please copy and paste the code instead of the example below.
% %%
% \begin{CCSXML}
% <ccs2012>
%    <concept>
%        <concept_id>10010147.10010257</concept_id>
%        <concept_desc>Computing methodologies~Machine learning</concept_desc>
%        <concept_significance>500</concept_significance>
%        </concept>
%  </ccs2012>
% \end{CCSXML}

% \ccsdesc[500]{Computing methodologies~Machine learning}

% \keywords{Mixture-of-Experts, Large Language Models, Optimization}

\maketitle

\section{Introduction}
Mixture-of-Experts (MoE) architectures have revolutionized deep models by reaching trillion-parameter sizes while maintaining computational efficiency \citep{smoe, switch, tang2025pangu}. The core principle -- routing inputs to a small subset of expert networks -- has proven transformative for large language and multimodal models \citep{deepseek, mixtral, llamamoe, multimodal}. 
However, uneven expert utilization has emerged as a critical challenge in MoE training \citep{soboleva2025router}. Some experts become overactivated while others remain underused, leading to inefficient use of model capacity. This imbalance becomes particularly problematic when experts are distributed across devices. Even if expert usage appears balanced over time, individual batches may trigger uneven activation across devices. This creates load imbalance, degrades parallelism, and increases communication overhead.

\begin{figure}[t]
    \centering    \includegraphics[width=\columnwidth]{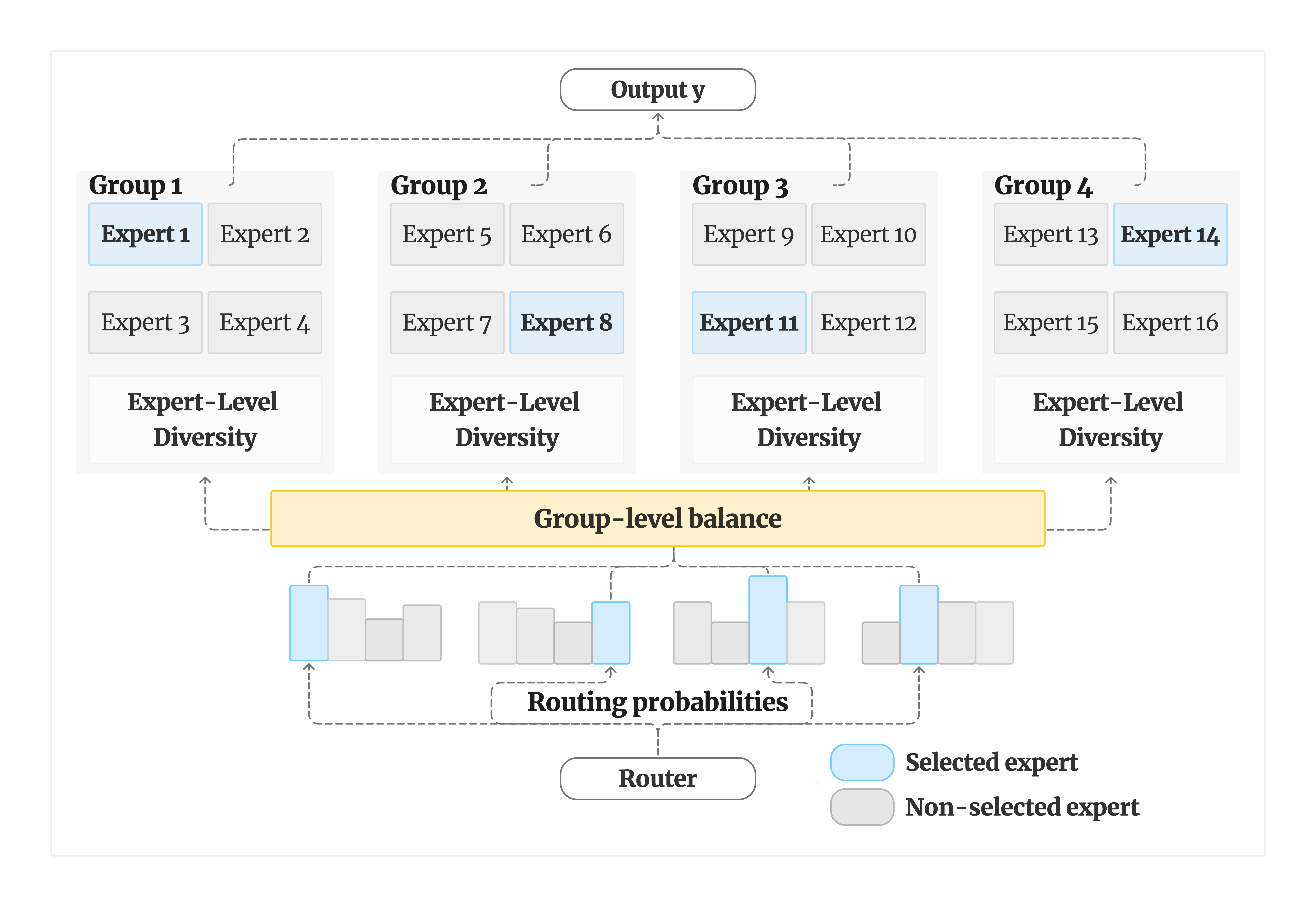}
    \caption{Diagram of the proposed \texttt{Hi-MoE}. Experts are organized into hierarchical groups, promoting complementary specialization within groups and balanced, device-level utilization across groups during routing.}
    \label{fig:diagram}
\end{figure}

Recent approaches have addressed this challenge by grouping experts and enforcing balanced activation through fixed expert selection from each group \citep{moge}. This strategy ensures more uniform computational loads across devices. Yet it comes at a cost: constraining selection to a fixed number of experts per group limits the diversity of reachable expert combinations compared to free selection from all available experts. This reveals a fundamental trade-off in MoE training between balanced expert utilization and meaningful specialization.

To address this challenge, we propose \texttt{Hi-MoE}, a hierarchical framework that combines explicit expert grouping with level-specific diversity and balancing mechanisms. While the idea of hierarchical MoE structures dates back to classic work on mixtures of experts \citep{jordan1994hierarchical} and has been revisited in modern settings \citep{hierarchical2, nguyen2024expert}, we adapt it to the context of grouped experts and large-scale sparse transformers. 
In particular, we merge recent grouped-expert designs with a two-level optimization hierarchy. At the intra-group level, we encourage complementary expert specialization, while at the inter-group level, we promote balanced utilization and prevent group-level collapse. Figure \ref{fig:diagram} illustrates the overall \texttt{Hi-MoE} architecture, where experts are partitioned into fixed groups and trained with distinct optimization strategies at each level.

As a result, \texttt{Hi-MoE} makes expert grouping an explicit optimization mechanism. This hierarchical view yields a simple, modular component that can be dropped into modern MoE architectures to obtain consistent gains in quality, convergence speed, and hardware efficiency.

\section{Related Work}
\paragraph{Mixture-of-Experts Foundations.} The idea of combining multiple specialized predictors traces back to early work on mixtures and ensembles \citep{de1989mixtures, jordan1991, jordan1994}, where the goal was to aggregate weak learners into a stronger model. Subsequent studies showed that individual components in a deep ensemble often capture complementary aspects of the data, and that ensemble effectiveness critically depends on maintaining diversity among its members \citep{allen2020towards}.
However, standard deep ensembles are expensive: they require training and storing many full models, and inference scales linearly with the ensemble size.

MoE addresses this bottleneck by sharing most parameters while activating only a small subset of expert networks per input. 
Formally, one MoE layer replaces the dense feed-forward block with a collection of $N$ expert networks $\{f_i\}_{i=1}^N$, where each expert is a function $f_i:\mathbb{R}^d \rightarrow \mathbb{R}^d$.
A gating network $g:\mathbb{R}^d \rightarrow \mathbb{R}^N$ maps an input representation $\x \in \mathbb{R}^d$ to routing logits, which are converted to probabilities
\[
\bm{\pi}(\x) = \mathrm{softmax}\bigl(g(\x)\bigr) \in \mathbb{R}^N.
\]
An MoE layer then activates only the $K$ most likely experts and combines their outputs.
We write the sparse mixture as
\begin{equation}
\mathrm{SparseMoE}(\x)
\;=\;
\sum_{i=1}^{N}
\Bigl(
\pi_i(\x)\,\mathbb{I}\{ i \in \mathrm{Top}\text{-}K(\bm{\pi}(\x)) \}
\Bigr)\, f_i(\x),
\label{eq:standard_moe}
\end{equation}
where $\mathrm{Top}\text{-}K(\bm{\pi}(\x))$ returns the indices of the $K$ largest entries of $\bm{\pi}(\x)$, and $\pi_i(\x)$ denotes the $i$-th component.

In practice, a small subset of experts often attracts most tokens while many experts remain almost idle, degrading both specialization and effective capacity.
This phenomenon, commonly referred to as \emph{routing collapse}, has been observed to cause representation collapse and severe load skew across experts in large-scale sparse models \citep{representationcollapse}.
From a systems perspective, such imbalances translate into stragglers and underutilized devices. From a modeling perspective, they undermine the very goal of MoE, namely to learn diverse expert behaviors.
In what follows, we review modern routing and load-balancing strategies, and distill the insights that motivate our hierarchical treatment of specialization and balance in \texttt{Hi-MoE}.

\begingroup
\renewcommand{\thefootnote}{}% убирает номер
\footnotetext{Sometimes the TopK operator in Equation (\ref{eq:standard_moe}) is applied before the softmax, but here we continue with this variant as it tends to perform slightly better.
% The option with the inversion order is discussed in the Appendix.
}
\addtocounter{footnote}{-1}% чтобы не увеличить счётчик
\endgroup

\paragraph{Loss-based load balancing.}
A standard way to counter routing collapse is to augment the training objective with an auxiliary load-balancing term $\mathcal{L}_{\text{load}}$ in addition to the task loss $\mathcal{L}_{\text{task}}$.
The goal of $\mathcal{L}_{\text{load}}$ is to encourage the router to spread tokens and routing probability mass across experts, rather than letting a few experts dominate.
GShard formulation \citep{gshard} defines the auxiliary loss as
\begin{equation}
    \label{eq:loss}
    \mathcal{L}_{\text{load}}
    \;=\;
    \alpha \cdot N\cdot \sum_{i=1}^N h_i \cdot P_i,
\end{equation}
where $h_i$ is the fraction of tokens dispatched to expert $i$ in a batch $\mathcal{B}$ of size $|\mathcal{B}|$,
$
h_i
=
\frac{1}{|\mathcal{B}|}
\sum_{\mathbf{x} \in \mathcal{B}}
\left[\operatorname{TopK}\big(\bm{\pi}(\mathbf{x})\big)\right]_i,
$
and $P_i$ is the mean routing probability for expert $i$,
$P_i
=
\frac{1}{|\mathcal{B}|}
\sum_{\mathbf{x} \in \mathcal{B}}
\pi_i(\mathbf{x}).$
The coefficient $\alpha$ controls the strength of this regularizer.
Minimizing \eqref{eq:loss} drives both the token histogram $\{h_i\}$ and the probability histogram $\{P_i\}$ toward uniformity, simultaneously regularizing which experts are selected and how confidently they are chosen.

Subsequent works refine this basic idea rather than replacing it.
Techniques such as logit normalization, temperature scaling, and per-layer or even per-token adaptive coefficients for $\mathcal{L}_{\text{load}}$ make the balancing signal more robust and easier to tune in deep stacks of MoE layers \citep{switch, skywork, glam, stmoe, openmoe, deepseekmoe, shen2024jetmoe}.
These advances reduce sensitivity to hyperparameters and improve training stability across layers, tasks, and batch regimes.

However, all loss-based approaches share a structural challenge: the router is optimized under a compound objective $\mathcal{L}_{\text{task}} + \mathcal{L}_{\text{load}}$.
As analyzed in \citet{lossfree}, the gradients from $\mathcal{L}_{\text{load}}$ can conflict with task gradients, distorting the optimization trajectory of both the router and the experts.
In practice, this interference can slow convergence or cap final quality, especially when strong balancing is required to prevent collapse.

\paragraph{Loss-free and biased load balancing.}
A more radical alternative is loss-free balancing. These approaches typically modify the dispatch rule itself, so balancing emerges from the routing mechanism rather than from an auxiliary objective \citep{lossfree, expertchoice}.
Biased strategies explicitly acknowledge that some experts may be legitimately more popular, and they try to balance only what matters for training and systems constraints, such as per-device throughput or worst-case stragglers \citep{megablocks, tutel}.
These trends highlight an important shift: instead of forcing all experts to be equally used, modern MoE training increasingly targets structured balance under practical constraints.
This motivates frameworks that separate (i) semantic specialization mechanisms from (ii) the specific notion of balance that is required by hardware and throughput.

\paragraph{Beyond flat MoE: grouped experts and hierarchical routing.}
Recent work therefore moves toward hardware-aware MoE architectures, where experts are partitioned into groups aligned with devices or shards.
It helps to equalize FLOPs per device and communication volume \citep{gracemoe, moetuner, moge}.
Closely related system frameworks explicitly optimize expert-to-GPU assignment to reduce cross-device transfers and alleviate load skew, often operating on groups rather than individual experts \citep{fastmoe, fastermoe, comet, tutel}.

While grouped routing and other hardware-aligned constraints reliably improve utilization by equalizing per-device FLOPs and communication, they do so by restricting routing freedom.
This restriction reduces the number of reachable expert combinations, which directly weakens diversity and limits conditional capacity.
Moreover, since experts in different groups are optimized on comparable token mixtures, they may converge to redundant functions.
As a result, grouping can replace expert-level collapse with a diversity collapse across groups: the system is balanced, yet many groups implement redundant behaviors.
A complementary line of work attacks redundancy in optimizer or weight space during fine-tuning.
For example, \citet{liu2024diversifying} use orthogonal expert updates to diversify MoE representations.
This is complementary to our setting: \texttt{Hi-MoE} acts directly in routing space during pre-training, before routing patterns collapse into redundant expert usage.

To address these challenges, we introduce \texttt{Hi-MoE}, a hierarchical framework that explicitly promotes complementary specialization within groups while preserving balanced group utilization at the system level.
While the idea of hierarchical MoE structures dates back to classic work on mixtures of experts \citep{jordan1994hierarchical} and has been revisited in modern settings \citep{hierarchical2, nguyen2024expert}, we adapt it to the context of grouped experts and large-scale sparse transformers.
In particular, at the intra-group level, we encourage complementary expert specialization, while promoting balanced utilization at the inter-group level.

\subsection*{Contributions}
Our contributions are summarized as follows:
\begin{itemize}
  \item We identify a fundamental balance--diversity tension in MoE routing: enforcing device-level balance via strong group constraints can stabilize utilization while increasing redundancy across experts and groups.

  \item We propose \texttt{Hi-MoE}, a hierarchical optimization framework that mitigates this tension through two complementary objectives: (i) an \emph{intra-group} diversity signal that discourages redundant co-activation and promotes complementary specialization, and (ii) an \emph{inter-group} balancing objective that equalizes device-aligned traffic and stabilizes training.

  \item \texttt{Hi-MoE} is a drop-in objective for standard top-$k$ MoE layers and routing pipelines, compatible with common balancing losses and routing constraints. It augments rather than replaces existing stabilization mechanisms.
\setcounter{footnote}{0}
  \item We demonstrate consistent gains from vision to language, including Swin-MoE variants on ImageNet, nanoGPT-1B pre-training on OpenWebText, and large-scale pre-training against OLMoE-7B. Across settings, \texttt{Hi-MoE} simultaneously improves load balance (lower CV) and boosts final quality. Our code is publicly available on GitHub\footnote{\url{https://github.com/brain-lab-research/Hi-MoE}.}.
\end{itemize}

\section{Hierarchical Mixture-of-Experts}
\paragraph{Setup.} We consider a Transformer block in which the dense FFN sub-layer is replaced by a grouped MoE module.
Let $\mathbf{x} \in \mathbb{R}^d$ denote the token representation in the residual stream.
For clarity, we write the FFN computation as a residual update,
\[
\mathbf{x}^{+} \;=\; \mathbf{x} \;+\; \mathrm{MoE}(\mathbf{x}),
\]
and omit standard components such as normalization and dropout, which are orthogonal to our routing design.

The MoE module contains $N$ experts $\{f_i\}_{i=1}^N$, where each expert is a feed-forward network $f_i:\mathbb{R}^d \rightarrow \mathbb{R}^d$.
In grouped MoE, experts are partitioned into $M$ disjoint groups
$\mathcal{G}_1, \ldots, \mathcal{G}_M$ (often aligned with devices),
and routing is constrained to activate a fixed number of experts \emph{within each group}, improving per-step device balance.
Concretely, the router produces logits $s(\mathbf{x}) \in \mathbb{R}^N$ and probabilities $p(\mathbf{x})=\mathrm{softmax}(s(\mathbf{x}))$.
For each group $m$, we select $K_m$ experts using a group-local Top-$K_m$ operator,
\begin{equation*}
\mathrm{MoE}(\mathbf{x})
\;=\;
\sum_{m=1}^{M}\;\sum_{i \in \mathcal{G}_m}
\Bigl(\, p_i(\mathbf{x})\cdot \mathbb{I}\{ i \in \mathrm{Top}\text{-}K_m(\mathbf{p}_{\mathcal{G}_m}(\mathbf{x})) \}\Bigr)\, f_i(\mathbf{x}),
% \label{eq:grouped-moe}
\end{equation*}
where $\mathbf{p}_{\mathcal{G}_m}(\mathbf{x}) := (p_i(\mathbf{x}))_{i\in\mathcal{G}_m}$ denotes the subvector of routing probabilities restricted to group $\mathcal{G}_m$.
A common choice is $K_m = K/M$, which enforces the same number of activated experts in every group.

\paragraph{Loss functions.} Our method relies on optimizing the task loss $\mathcal{L}_{\text{task}}$ together with a load-balancing objective $\mathcal{L}_{\text{load}}$, augmented by the hierarchical components introduced below.
This combination is important for stable routing optimization in hierarchical MoE: $\mathcal{L}_{\text{task}}$ provides the semantic learning signal, $\mathcal{L}_{\text{load}}$ prevents degeneracies in sparse dispatch, and our regularizers shape \emph{where} balance and diversity are enforced.
While we use the GShard-style definition of $\mathcal{L}_{\text{load}}$ in \eqref{eq:loss}, our framework does not depend on this particular choice and is compatible with alternative load-balancing terms.

\paragraph{Routing distribution.} For any token representation $\mathbf{x}$, we consider a hierarchical routing distribution that factorizes over groups and experts,
\begin{align}
p(e \mid \mathbf{x})
\;=\;
p(g \mid \mathbf{x})\, p(e \mid \mathbf{x}, g),
\qquad e \in \mathcal{G}_g .
\label{eq:hier_prod}
\end{align}
This factorization is natural in grouped MoE: $p(g\mid \mathbf{x})$ captures how much routing mass is assigned to each group (e.g., a GPU shard), while $p(e\mid \mathbf{x}, g)$ captures specialization \emph{within} the selected group.
Expanding the squared $\ell_2$ concentration of the global routing distribution yields
\begin{align}
\|p(\cdot \mid \mathbf{x})\|_2^2
&\;=\;
\sum_{e=1}^N p(e \mid \mathbf{x})^2
\;=\;
\sum_{g=1}^M p(g \mid \mathbf{x})^2 \sum_{e \in \mathcal{G}_g} p(e \mid \mathbf{x}, g)^2 \notag\\
&\;=\;
\sum_{g=1}^M p(g \mid \mathbf{x})^2 \, \|p(\cdot \mid \mathbf{x}, g)\|_2^2 .
\label{eq:l2_decomp}
\end{align}
We define the induced group mass as the expert marginal aggregated within each group,
\begin{align*}
p(g \mid \mathbf{x})
\;\triangleq\;
\sum_{e \in \mathcal{G}_g} p(e \mid \mathbf{x}),
% \label{eq:group_mass}
\end{align*}
and for a batch $\mathcal{B}=\{\mathbf{x}_b\}_{b=1}^B$, the corresponding (soft) group load
\begin{align*}
L_g
\;\triangleq\;
\sum_{b=1}^B p(g \mid \mathbf{x}_b).
% \label{eq:group_load}
\end{align*}
Since $L_g$ aggregates how much probability mass is routed to group $g$, balancing $\{L_g\}_{g=1}^M$ is the appropriate notion of GPU utilization, even when the global expert concentration $\|p(\cdot \mid \mathbf{x})\|_2^2$ is fixed.

Crucially, the decomposition \eqref{eq:hier_prod} separates two sources of concentration.
For a fixed token $\mathbf{x}$, modifying how probability mass is \emph{spread across groups} changes the quadratic term $\sum_{g=1}^M p(g \mid \mathbf{x})^2$, which serves as a direct proxy for group-level imbalance.
Among all distributions over $M$ groups, this quantity is minimized by the uniform assignment $p(g \mid \mathbf{x})=1/M$. Therefore, if our goal is to improve per-device utilization, it is beneficial to explicitly penalize group-level concentration.
This motivates the inter-group regularization introduced next.

\subsection{Inter-group regularization}
Grouped routing makes GPU utilization largely depend on how routing mass is distributed across groups.
Under the factorization in \eqref{eq:hier_prod}, group imbalance is captured by the concentration of $p(g\mid \mathbf{x})$, e.g., through the quadratic proxy $\sum_g p(g\mid \mathbf{x})^2$.
In practice, however, the executed computation is sparse: after routing, only Top-$K$ experts are actually dispatched.
We therefore regularize the post-selection routing weights, focusing the signal on the experts that incur real compute and communication.

Let $\bm{\pi}(\mathbf{x}) \in \mathbb{R}^N$ denote the router probabilities and define the selected (sparsified) weights
\begin{equation}
\widetilde{\boldsymbol{\pi}}(\mathbf{x})
\;:=\;
\bigoplus_{m=1}^{M}
\left(
\mathbf{p}_{\mathcal{G}_m}(\mathbf{x})
\odot
\mathbf{1}_{\mathrm{Top}\text{-}K_m\!\left(\mathbf{p}_{\mathcal{G}_m}(\mathbf{x})\right)}
\right),
\label{eq:topk_pi}
\end{equation}
where $\mathbf{1}_{\mathrm{Top}\text{-}K_m(\cdot)}$ is the binary Top-$K_m$ indicator mask and non-selected entries are set to zero.
We then introduce the inter-group regularizer
\begin{equation}
\mathcal{R}_{\text{inter}}
\;=\;
\lambda_{\text{inter}} \,\bigl\|\widetilde{\bm{\pi}}(\mathbf{x})\bigr\|_2^2 .
\label{eq:r_inter}
\end{equation}
Minimizing $\|\widetilde{\bm{\pi}}\|_2^2$ discourages overly concentrated routing among the selected experts, which, through \eqref{eq:l2_decomp}, reduces the contribution of group-level concentration and improves uniform utilization across device-aligned groups.
This term consistently lowers group and expert load variance, yielding more even GPU work per step.

At the same time, stronger inter-group equalization has a predictable side effect: it restricts routing freedom.
When tokens are forced to spread more uniformly across groups, the model explores fewer cross-group expert combinations, and experts in different groups can become functionally redundant.
Intuitively, independent groups trained under similar token mixtures may learn similar decision boundaries, so the same token types activate analogous experts in multiple groups.
Indeed, in \eqref{eq:hier_prod}, $\mathcal{R}_{\text{inter}}$ primarily shapes $p(g\mid \mathbf{x})$, but it does not ensure that the conditional distributions $p(e\mid \mathbf{x}, g)$ become complementary across groups.
This motivates an additional intra-group mechanism that explicitly promotes specialization where it is most expressive: within each group.

\subsection{Intra-group regularization}
To counter cross-group redundancy and recover conditional capacity, we encourage experts within each group to diversify their assignments.
A proxy for diversity is the (negative) concentration of the routing distribution: for a fixed token, peaked probabilities correspond to committing to a small subset of experts, while flatter probabilities correspond to mixing many experts.

Concretely, we introduce an anti-concentration regularizer on the pre-selection routing distribution,
\begin{equation}
\mathcal{R}_{\text{intra}}
\;=\;
-\,\lambda_{\text{intra}} \,\bigl\|\bm{\pi}(\mathbf{x})\bigr\|_2^2 .
\label{eq:r_intra}
\end{equation}
Minimizing \eqref{eq:r_intra} increases $\|\bm{\pi}(\mathbf{x})\|_2^2$, thereby reducing within-token overlap and encouraging the router to commit more strongly to a subset of experts.
In isolation, this would risk expert collapse.
In \texttt{Hi-MoE}, however, it is paired with $\mathcal{L}_{\text{load}}$ and $\mathcal{R}_{\text{inter}}$, which keep global expert usage and group traffic balanced while still allowing decisive specialization inside each group.
We emphasize that this is not meant to replace established balancing recipes.
Rather, \eqref{eq:r_intra} is a lightweight, level-specific signal that complements $\mathcal{L}_{\text{task}}$ and $\mathcal{L}_{\text{load}}$ by shaping specialization inside each group.

On its own, intra-group diversification can drift toward degenerate regimes, where a few experts dominate because the router discovers locally sharp partitions.
We therefore pair \eqref{eq:r_intra} with a simple bias-corrected routing computation that discourages persistent over-selection of historically preferred experts,
\begin{equation*}
\bm{\pi}(\mathbf{x})
\;=\;
\mathrm{softmax}\!\left(\frac{\mathbf{g}(\mathbf{x}) - \tau \,\overline{\mathbf{g}}}{T}\right),
\qquad
\overline{\mathbf{g}} \leftarrow \beta \overline{\mathbf{g}} + (1-\beta)\mathbf{g}(\mathbf{x}),
% \label{eq:biascorr}
\end{equation*}
where $\overline{\mathbf{g}}$ is an exponential moving average of router logits.
This correction is a practical stabilizer: it nudges routing away from entrenched modes without introducing another explicit balancing loss, and it preserves the modularity of our framework.

\subsection{Complete objective and evaluation protocol}
The \texttt{Hi-MoE} training objective combines task learning, standard load balancing, and the proposed hierarchical regularizers:
\begin{equation}
\mathcal{L}
\;=\;
\mathcal{L}_{\text{task}}
\;+\;
\mathcal{L}_{\text{load}}
\;+\;
\mathcal{R}_{\text{intra}}
\;+\;
\mathcal{R}_{\text{inter}}.
\label{eq:full_objective}
\end{equation}
In the following, we quantify load balance using the coefficient of variation (CV) of expert loads,
\begin{equation}
\mathrm{CV}
\;=\;
\frac{\sigma_{\text{load}}}{\mu_{\text{load}}},
\label{eq:load}
\end{equation}
where $\sigma_{\text{load}}$ and $\mu_{\text{load}}$ are the standard deviation and mean of token counts per expert.
Lower CV indicates more uniform utilization.
This is the standard MoE load-balance metric used since \citet{smoe}; \citet{lossfree} report a closely related max-mean imbalance metric, reflecting the same goal of measuring load concentration.
\begin{figure}[ht] \centering \includegraphics[width=\columnwidth]{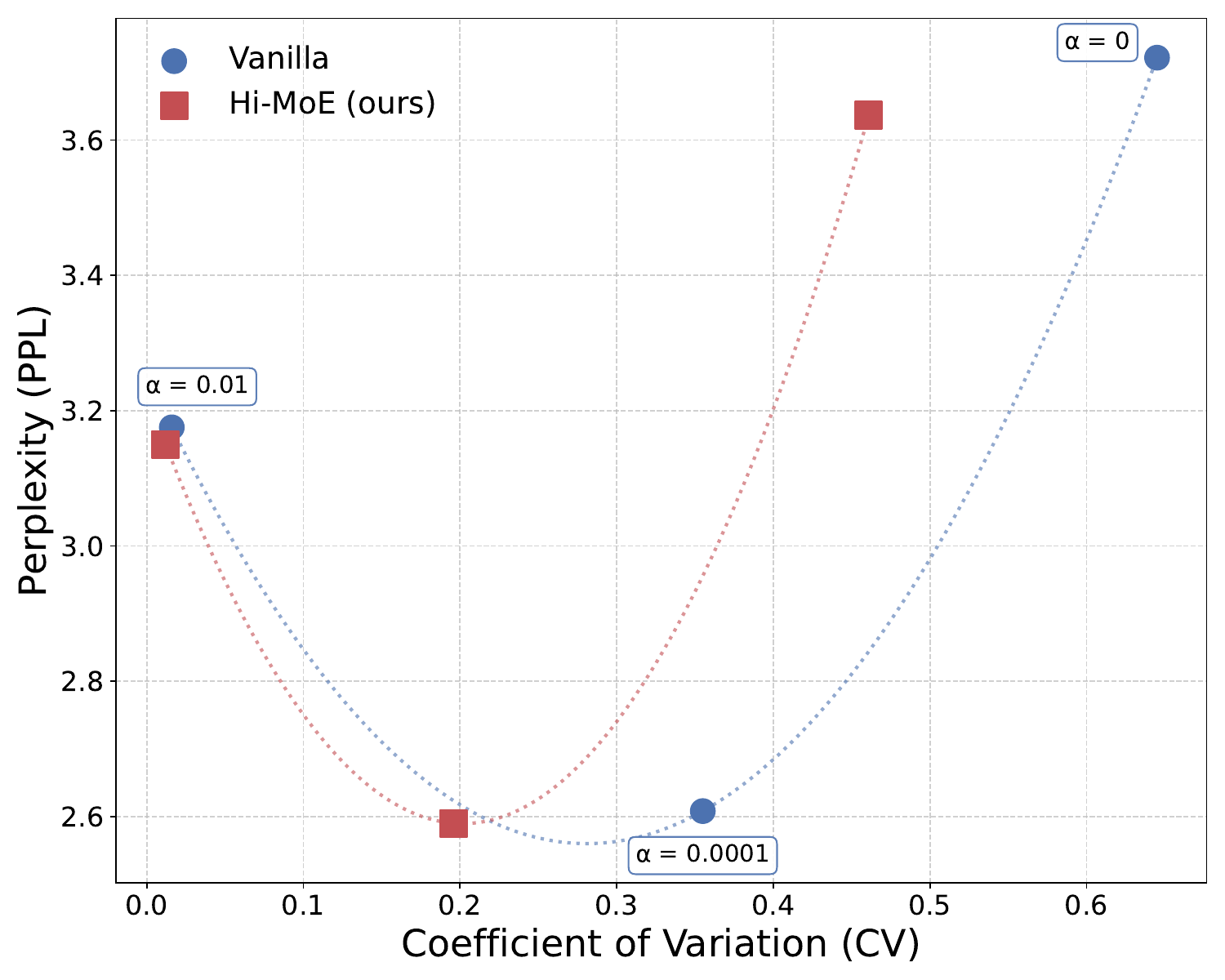} \caption{PPL-CV trade-off.} \label{fig:tradeoff} \end{figure}
Figure~\ref{fig:tradeoff} reports the perplexity--CV trade-off and shows that \texttt{Hi-MoE} expands the Pareto frontier relative to grouped baselines (obtained with $\lambda_{\text{intra}}=\lambda_{\text{inter}}=0$), improving balance without sacrificing quality, and enabling predictable tuning between the two.

\section{Analysis}\label{sec:main:theory}
\texttt{Hi-MoE} is designed to satisfy two requirements that are \emph{simultaneously} important at scale:
(i) \emph{systems-aware balance} (uniform GPU/group utilization), and
(ii) \emph{functional diversity} (experts do not collapse to redundant behaviors).
A key message is that the hierarchical regularizers in \eqref{eq:full_objective} implement a principled \emph{constrained} view of MoE training: maximize specialization subject to balance (hardware constraints). Below we formalize this statement.

Recall that experts are partitioned into $M$ fixed groups $\{\mathcal{G}_g\}_{g=1}^M$.
For a token representation $\mathbf{x}$, let $\widetilde{\bm{\pi}}(\mathbf{x})\in\mathbb{R}^N$ be the post-Top-$K$ routing weights defined in \eqref{eq:topk_pi}.
Define the induced group assignment (a distribution over groups):
\[
r(\mathbf{x})=(r_1(\mathbf{x}),\dots,r_M(\mathbf{x}))\in\Delta^{M-1},
\qquad
r_g(\mathbf{x})
\;:=\;
\frac{\sum_{e\in\mathcal{G}_g}\widetilde{\pi}_e(\mathbf{x})}{\sum_{e=1}^N \widetilde{\pi}_e(\mathbf{x})}.
\]
Let $S_{\max}:=\max_{g\in[M]}|\mathcal{G}_g|$ denote the maximum group size.
For a random token $\mathbf{X}$, define the mean group-load distribution $\bar r:=\mathbb{E}[r(\mathbf{X})]\in\Delta^{M-1}$.
For the reader’s convenience, we provide Table \ref{tab:notation} in Appendix \ref{sec:theory}, which summarizes \emph{all notation used in this paper}.
%-------------------------------------------------------------------------------
\subsection{Inter-group regularizer as a certified upper bound on group imbalance}
\label{sec:main:theory-inter}

\begin{theorem}[Inter-group regularizer controls group CV]
\label{thm:main:inter-controls-cv}
Assume $r(\mathbf{x})$ is normalized so that $\sum_{g=1}^M r_g(\mathbf{x})=1$ for all tokens.
Let $\mathbf{X}$ be a random token representation.
Define the \emph{mean group load distribution} $\bar r := \mathbb{E}[r(\mathbf{X})]\in\Delta^{M-1}$.
Then
\[
\|\bar r\|_2^2
\;\le\;
\mathbb{E}\bigl[\|r(\mathbf{X})\|_2^2\bigr]
\;\le\;
S_{\max}\,\mathbb{E}\bigl[\|\widetilde{\bm{\pi}}(\mathbf{X})\|_2^2\bigr].
\]
Consequently, for any batch whose (soft) group-load vector is proportional to $\bar r$, its group-level coefficient of
variation satisfies
\[
\mathrm{CV}_{\text{group}}^2 \;+\; 1
\;=\;
M\|\bar r\|_2^2
\;\le\;
M S_{\max}\,\mathbb{E}\bigl[\|\widetilde{\bm{\pi}}(\mathbf{X})\|_2^2\bigr].
\]
\end{theorem}

\begin{tcolorbox}[colback=gray!0,colframe=gray!30,arc=2mm]
\paragraph{Discussion.}
Since $\mathcal{R}_{\text{inter}}$ in \eqref{eq:r_inter} is exactly proportional to
$\|\widetilde{\bm{\pi}}(\mathbf{x})\|_2^2$,
Theorem~\ref{thm:main:inter-controls-cv} provides a \emph{certified} link between the inter-group term and group/GPU imbalance.
This formalizes why $\mathcal{R}_{\text{inter}}$ is not an ad-hoc trick:
it upper-bounds the same $\ell_2$-concentration that is CV.
\end{tcolorbox}
Proofs of this theorem and the subsequent theory are provided in Appendix \ref{sec:theory}.
%-------------------------------------------------------------------------------
\subsection{Anti-overlap increases expert diversity by decoupling gradients}
\label{sec:main:theory-gradients}

To connect overlap to what changes in training dynamics, we analyze gradient coupling.
Let each expert have parameters $\theta_i\in\mathbb{R}^P$ (same dimension $P$ for all experts).
For a token $\mathbf{x}$ and its associated loss contribution, denote the expert-local gradient by
$g_i(\mathbf{x}) := \nabla_{\theta_i}\ell_i(\mathbf{x};\theta_i)\in\mathbb{R}^P$.
The MoE weighting yields an effective update direction proportional to
$\pi_i(\mathbf{x})\,g_i(\mathbf{x})$.

\begin{theorem}[Overlap bounds cross-expert gradient coupling]
\label{thm:main:grad-coupling}
Assume $\|g_i(\mathbf{x})\|_2 \le G$ for all $i$ and all $\mathbf{x}$ (clipping scenario).
Define the total (soft) cross-expert coupling on a token by
\[
C(\mathbf{x})
\;:=\;
\sum_{i\neq j}\bigl|\langle \pi_i(\mathbf{x})g_i(\mathbf{x}),\,\pi_j(\mathbf{x})g_j(\mathbf{x})\rangle\bigr|.
\]
Then
\[
C(\mathbf{x})
\;\le\;
G^2\sum_{i\neq j}\pi_i(\mathbf{x})\pi_j(\mathbf{x})
\;=\;
G^2\bigl(1-\|\bm{\pi}(\mathbf{x})\|_2^2\bigr).
\]
Therefore, maximizing $\|\bm{\pi}(\mathbf{x})\|_2^2$ (i.e., applying $\mathcal{R}_{\text{intra}}$)
provably decreases an upper bound on how strongly different experts are trained on the \emph{same} token update.
\end{theorem}

\begin{tcolorbox}[colback=gray!0,colframe=gray!30,arc=2mm]
\paragraph{Discussion.}
Theorem~\ref{thm:main:grad-coupling} gives a direct mechanistic explanation:
\emph{anti-overlap reduces gradient alignment/coupling between experts,}
which increases the chance that experts follow different optimization trajectories and learn complementary functions.
\end{tcolorbox}
%-------------------------------------------------------------------------------
\subsection{Intra-group regularizer: maximizing $\|\pi\|_2^2$ increases routing information}
\label{sec:main:theory-mi}

To formalize coverage of experts, we use a standard specialization proxy:
how informative the expert identity is about the input token.
Consider stochastic routing where $E\sim \pi(\mathbf{X})$ given a random token $\mathbf{X}$.
A quantity tightly connected to $\|\pi\|_2^2$ is the collision conditional entropy, which decreases when routing becomes more decisive:
\[
H_2(E\mid \mathbf{X})
\;:=\;
-\log \mathbb{E}_{\mathbf{X}}\Bigl[\sum_{i=1}^N \pi_i(\mathbf{X})^2\Bigr]
\;=\;
-\log \mathbb{E}_{\mathbf{X}}\bigl[\|\pi(\mathbf{X})\|_2^2\bigr].
\]

\begin{theorem}[Balanced marginals $\Rightarrow$ $\|\pi\|_2^2$ maximizes collision mutual information]
\label{thm:main:renyi-mi}
Let $p(i):=\mathbb{E}_{\mathbf{X}}[\pi_i(\mathbf{X})]$ be the marginal expert usage probability.
Define the collision mutual information
\[
I_2(\mathbf{X};E)
\;:=\;
H_2(E) - H_2(E\mid \mathbf{X}),
\qquad
H_2(E):=-\log\sum_{i=1}^N p(i)^2.
\]
If expert usage is perfectly balanced, i.e.\ $p(i)=1/N$ for all $i$, then
\[
I_2(\mathbf{X};E)
\;=\;
\log\Bigl(N\,\mathbb{E}_{\mathbf{X}}[\|\pi(\mathbf{X})\|_2^2]\Bigr),
\]
which is strictly increasing in $\mathbb{E}[\|\pi(\mathbf{X})\|_2^2]$.
\end{theorem}
\begin{corollary}[Why \texttt{Hi-MoE} is strictly better than balance-only on specialization]
\label{cor:balance-only-mi}
Under balanced marginals $p(i)=1/N$, the \emph{least} specialized conditional routing is token-independent uniform routing
$\pi(\mathbf{x})\equiv (1/N)\mathbf{1}$, which yields $\mathbb{E}\|\pi(\mathbf{X})\|_2^2 = 1/N$ and hence $I_2(\mathbf{X};E)=0$.
Any increase in $\mathbb{E}\|\pi(\mathbf{X})\|_2^2$ (as induced by $\mathcal{R}_{\text{intra}}$)
strictly increases $I_2(\mathbf{X};E)$, i.e.\ increases specialization/coverage.
\end{corollary}
\begin{tcolorbox}[colback=gray!0,colframe=gray!30,arc=2mm]
\paragraph{Discussion.}
Corollary~\ref{cor:balance-only-mi} precisely captures the intended message:
\emph{for the same global balance}, \texttt{Hi-MoE} pushes the router toward higher expert--input mutual information,
which is a rigorous proxy for experts covering different parts of the input space.
\end{tcolorbox}
%-------------------------------------------------------------------------------
\subsection{\texttt{Hi-MoE} objective}
\label{sec:main:theory-lagrangian}
Define two measurable costs:
\[
\underbrace{\mathcal{C}_{\text{sys}} := \mathbb{E}\bigl[\|\widetilde{\bm{\pi}}(\mathbf{X})\|_2^2\bigr]}_{\text{proxy for group/GPU imbalance }},
\qquad
\underbrace{\mathcal{C}_{\text{ov}} := \mathbb{E}\bigl[1-\|\bm{\pi}(\mathbf{X})\|_2^2\bigr]}_{\text{routing overlap}}.
\]
Then the following constrained problem directly encodes our intent:
\begin{equation}
\min \;\; \mathcal{L}_{\text{task}} + \mathcal{L}_{\text{load}}
\quad\text{s.t.}\quad
\mathcal{C}_{\text{sys}}\le \varepsilon_{\text{sys}},
\;\;
\mathcal{C}_{\text{ov}}\le \varepsilon_{\text{ov}}.
\label{eq:main:constrained}
\end{equation}
Finally, the Lagrangian relaxation of \eqref{eq:main:constrained} is
\[
\mathcal{L}_{\text{task}}+\mathcal{L}_{\text{load}}
+\lambda_{\text{inter}}\,\mathbb{E}\|\widetilde{\bm{\pi}}(\mathbf{X})\|_2^2
+\lambda_{\text{intra}}\,\mathbb{E}\bigl[1-\|\bm{\pi}(\mathbf{X})\|_2^2\bigr],
\]
which is equivalent (up to an additive constant $\lambda_{\text{intra}}$) to
\[
\mathcal{L}_{\text{task}}+\mathcal{L}_{\text{load}}
+\mathbb{E}\Bigl[\lambda_{\text{inter}}\|\widetilde{\bm{\pi}}(\mathbf{X})\|_2^2 - \lambda_{\text{intra}}\|\bm{\pi}(\mathbf{X})\|_2^2\Bigr],
\]
i.e.\ to adding $\mathcal{R}_{\text{inter}}$ and $\mathcal{R}_{\text{intra}}$ as in \eqref{eq:r_inter}--\eqref{eq:r_intra}.

\begin{tcolorbox}[colback=gray!0,colframe=gray!30,arc=2mm]
\paragraph{Conclusion.}
We show that the two coefficients
$\lambda_{\text{inter}},\lambda_{\text{intra}}$ are interpretable as \emph{Lagrange multipliers} of a principled
balance--specialization constrained optimization problem.
Combined with Theorems~\ref{thm:main:inter-controls-cv}, \ref{thm:main:grad-coupling}, and \ref{thm:main:renyi-mi}, this yields a coherent explanation:
\texttt{Hi-MoE} enforces hardware-aligned balance while increasing specialization (reducing overlap and increasing routing information).
\end{tcolorbox}

\begin{figure*}[t]
\centering
\includegraphics[width=0.9\textwidth]{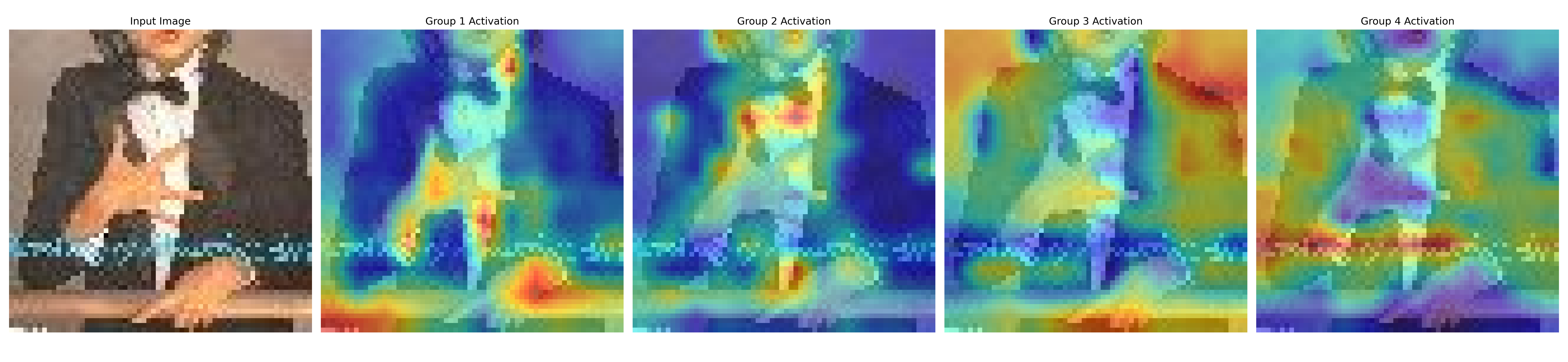}
\caption{Group-level attention patterns of Swin Transformer with \texttt{Hi-MoE} on Tiny ImageNet. Different expert groups automatically specialize on distinct visual features: Group 1 (hands/rail), Group 2 (light-colored objects, e.g., white shirt), Group 3 (dark-colored objects, e.g. suit, surfaces behind human), Group 4 (body parts).}
\label{fig:swin_attention}
\end{figure*}
\section{Experiments}
\label{sec:experiments}
% This section is constructed as follows. 

We evaluate \texttt{Hi-MoE} across a diverse set of domains, datasets, and model scales. Concretely, we consider three pre-training settings: (i) computer vision with Swin Transformer with MoE FFNs (1B) \citep{swin} on Tiny ImageNet \citep{imagenet}, (ii) language modeling with nanoGPT \citep{nanogpt} on OpenWebText \citep{openwebtext}, and (iii) large-scale validation with OLMoE-7B \citep{olmoe} on the Dolma dataset mixture \citep{dolma}.

For all models, we use a hierarchical organization with four expert groups. For Swin-MoE, each group contains 16 experts (64 total), with one expert selected per group (4 total). The nanoGPT model contains eight experts organized into four groups of two, selecting one expert per group (4 total). OLMoE-7B uses 64 experts organized into four groups with two experts selected per group (8 total).

We compare against Vanilla MoE (a non-group counterpart, based on a GShard-style router), MoE with loss-free load balancing \citep{lossfree}, ST-MoE \citep{stmoe}, and MoGE \citep{moge}, a well-established grouped MoE approach. We exclude shared experts for the reasons described in \citep{olmoe}: they reduce the number of possible expert combinations by nearly 90\% while offering no performance benefits over fully routed experts. Given the restrictions imposed by grouping, we prioritize flexibility by allowing all experts to be routed.

We reserve a validation set from the training corpus to evaluate both model quality and load balancing, and a test set for final measurement. We report standard performance metrics -- perplexity (PPL) and accuracy (Acc) -- and assess load balance using the coefficient of variation (CV) of expert loads, which captures the relative dispersion of routing decisions within an MoE layer (Eq.~\eqref{eq:load}).

Other reproducibility details can be found in Appendix \ref{app:repro}.

\subsection{Swin-MoE on Tiny ImageNet pre-training.}
We first evaluate \texttt{Hi-MoE} in a computer-vision setting, measuring both accuracy and routing balance.
We replace the feed-forward layers of the Swin Transformer with \texttt{Hi-MoE} modules, keeping the standard Swin hierarchy and sparsifying 10 blocks in total (9 out of 18 blocks in the third stage and 1 out of 2 blocks in the fourth stage).
The final model has approximately 1B total parameters, while using 83M \emph{active} parameters per token due to sparse expert routing.
We train all variants for 200 epochs on Tiny ImageNet and report both classification accuracy and coefficient of variation (CV) as in \eqref{eq:load}.

\paragraph{Performance metrics.} The final classification performance is presented in Table~\ref{tab:swin_results}. 
\begin{table}[!htbp]
\caption{Final performance on Tiny ImageNet classification. Results are reported as mean$\pm$std over 3 runs with different random seeds.}
\label{tab:swin_results}
\centering
\setlength{\tabcolsep}{3pt}
\begin{tabular}{lccc}
\toprule
\textbf{Method} & \textbf{Acc@1 (\%)} & \textbf{Acc@5 (\%)} & \textbf{CV}  \\
\midrule
Vanilla GShard-style & $58.61\pm 0.24$ & $80.39\pm 0.21$ & $0.33\pm 0.02$ \\
Loss-free balancing & $58.87\pm 0.29$ & $79.65\pm 0.26$ & $0.31\pm 0.02$ \\
MoGE & $58.61\pm 0.27$ & $80.92\pm 0.23$ & $0.30\pm 0.02$ \\
\midrule
\textbf{\texttt{Hi-MoE}} & $\textbf{59.13}\pm \textbf{0.19}$ & $\textbf{81.18}\pm \textbf{0.17}$ & $\textbf{0.29}\pm \textbf{0.01}$ \\
\bottomrule
\end{tabular}
\end{table}
\texttt{Hi-MoE} outperforms all baselines in accuracy while also achieving the best load balance. This is consistent with our hierarchical objective, which encourages activating a broader, more complementary set of experts than vanilla GShard-style or loss-free routing: spreading traffic improves coverage and lets experts specialize over more distinct regions of the input space. Compared to grouped routing, our intra-group anti-overlap term further acts as a diversity regularizer by explicitly discouraging persistent co-activation, increasing within-group heterogeneity and reducing redundancy. Consequently, even when overall load balancing is similar (and grouped methods already balance better than vanilla or loss-free routing), \texttt{Hi-MoE} retains a consistent edge.

% \begin{table}[H]
% \centering
% \setlength{\tabcolsep}{3pt}
% \begin{tabular}{lccc}
% \toprule
% \textbf{Method} & \textbf{Acc@1 (\%)} & \textbf{Acc@5 (\%)} & \textbf{CV}  \\
% \midrule
% Vanilla & $58.61\pm 0.24$ & $80.39\pm 0.21$ & $0.31\pm 0.018$ \\
% Lossfree & $58.87\pm 0.29$ & $79.65\pm 0.26$ & $0.30\pm 0.019$ \\
% MoGE & $58.61\pm 0.27$ & $80.92\pm 0.23$ & $0.30\pm 0.020$ \\
% \midrule
% \textbf{\texttt{Hi-MoE}} & $\textbf{58.93}\pm \textbf{0.19}$ & $\textbf{80.98}\pm \textbf{0.17}$ & $\textbf{0.29}\pm \textbf{0.015}$ \\
% \bottomrule
% \end{tabular}
% \caption{Final performance on Tiny ImageNet classification. Results are reported as mean$_{\pm \text{std}}$ over 3 runs with different random seeds.}
% \label{tab:swin_results}
% \end{table}
\paragraph{Specialization.} We further probe specialization by visualizing group-level attention maps. Figure~\ref{fig:swin_attention} shows that, despite the hard partition of experts into fixed groups, \texttt{Hi-MoE} learns a clear division of labor: each group concentrates on a different subset of visual cues (e.g., body parts vs. foreground objects vs. background surfaces).
This behavior is exactly what the anti-overlap regularizer encourages -- it pushes routing to be more decisive and reduces cross-group co-activation, allowing groups to follow distinct optimization trajectories and specialize rather than collapse to redundant features.

\paragraph{Load balance.} Figure~\ref{fig:swin_expert_freq} compares the expert activation frequency distribution in the 7th MoE layer of Swin-MoE after training on Tiny ImageNet.
\begin{figure}[!htbp]
    \centering
    \includegraphics[width=0.85\columnwidth]{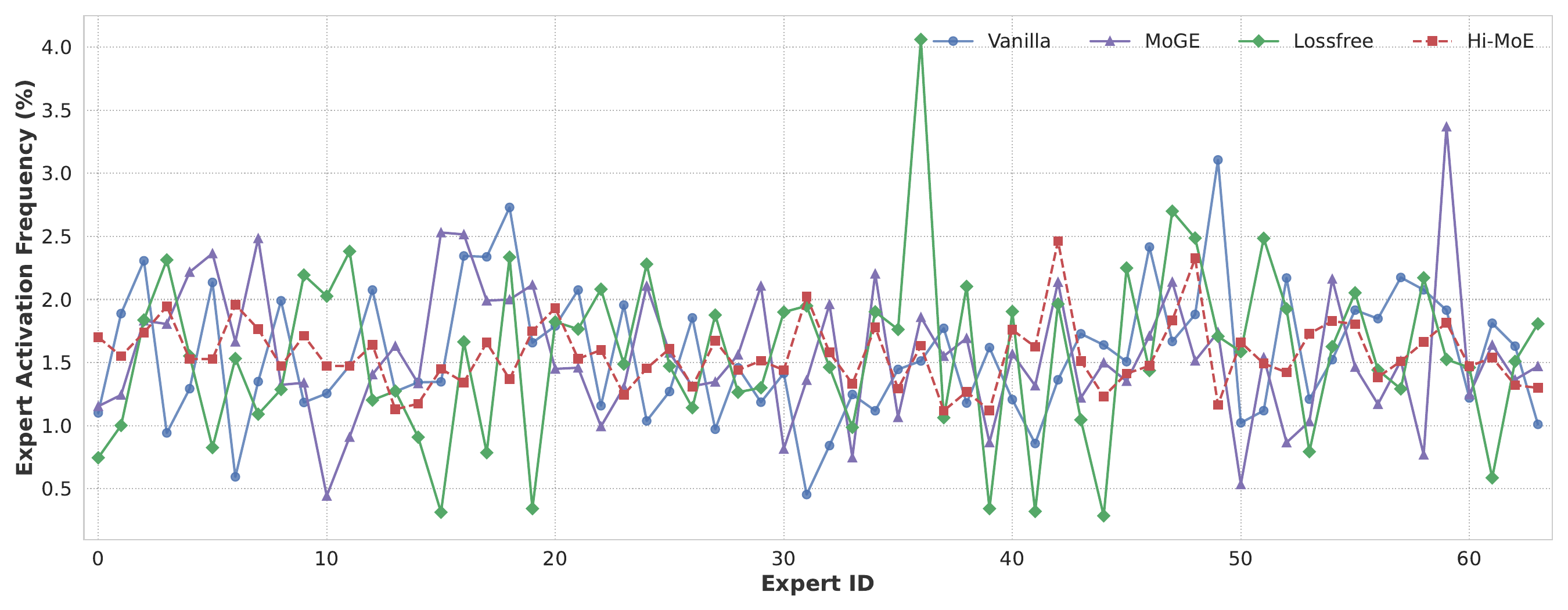}
    \caption{Expert activation frequency distribution in the 7th MoE layer (Third Stage) of Swin-MoE after training on Tiny ImageNet.}
    \label{fig:swin_expert_freq}
\end{figure}
Recall that we use 4 expert groups with 16 experts per group and target near-uniform activation. Vanilla GShard-style MoE is heavily skewed: $3$ of the $4$ most active experts land on the second GPU, creating a hotspot. Loss-free routing exhibits the largest imbalance in activation frequencies, with a few experts over-selected and many under-used, indicating the strongest capacity imbalance and likely explaining its weaker accuracy. MoGE improves inter-group balance but still exhibits intra-group collapse (e.g., on the last GPU one expert dominates while several receive too few tokens). In contrast, our method achieves a more uniform pattern.

\subsection{nanoGPT on OpenWebText.}
We next move from vision to language modeling and evaluate \texttt{Hi-MoE} in an NLP setting using nanoGPT trained on OpenWebText. We integrate \texttt{Hi-MoE} into 12 MoE layers within a 24-layer transformer. The resulting model contains approximately 1.06B total parameters, with 507M parameters active per token. We train the model on OpenWebText for 49B tokens.

\paragraph{Performance metrics.} We begin with the quality metrics presented in Table \ref{tab:combined_results}.
Overall, \texttt{Hi-MoE} achieves the best trade-off between quality and utilization: it improves final model performance while maintaining strong expert balance.

\begin{table}[!htbp]
\caption{Comparison of language modeling performance across nanoGPT architectures. Test metrics are reported as mean$\pm$std over 3 runs with different random seeds.}
\label{tab:combined_results}
\centering
\setlength{\tabcolsep}{4pt}
\begin{tabular}{lccc}
\toprule
\textbf{Method} & \textbf{PPL} & \textbf{CV}\\
\midrule
Vanilla GShard-style & $2.985\pm 0.024$ & $0.31\pm 0.016$ \\
Loss-free balancing & $2.979\pm 0.027$ & $0.37\pm 0.017$ \\
ST-MoE & $2.981\pm 0.022$ & $0.33\pm 0.016$ \\
MoGE& $2.977\pm 0.025$ & $0.25\pm 0.018$ \\
\midrule
\textbf{\texttt{Hi-MoE}} & $\textbf{2.947}\pm \textbf{0.019}$ & $\textbf{0.23}\pm \textbf{0.013}$ \\
\bottomrule
\end{tabular}
\end{table}

\paragraph{Component-wise ablation.} To isolate the contribution of each hierarchical component, Table~\ref{tab:nanogpt_ablation} reports an ablation on nanoGPT-1B. The pattern is consistent: $\mathcal{R}_{\text{inter}}$ primarily improves balance, $\mathcal{R}_{\text{intra}}$ primarily improves quality, and the bias correction acts as a stabilizer rather than the source of the gains.

\begin{table}[!htbp]
\caption{Component-wise ablation of \texttt{Hi-MoE} on nanoGPT-1B.}
\label{tab:nanogpt_ablation}
\centering
\setlength{\tabcolsep}{5pt}
\begin{tabular}{lcc}
\toprule
\textbf{Config} & \textbf{PPL} & \textbf{CV} \\
\midrule
MoGE & 2.977 & 0.25 \\
+ $\mathcal{R}_{\text{intra}}$ & 2.960 & 0.27 \\
+ $\mathcal{R}_{\text{inter}}$ & 2.972 & 0.21 \\
+ bias correction & 2.976 & 0.25 \\
+ $\mathcal{R}_{\text{intra}} + \mathcal{R}_{\text{inter}}$ & 2.950 & 0.23 \\
\midrule
\textbf{\texttt{Hi-MoE}} & \textbf{2.947} & \textbf{0.23} \\
\bottomrule
\end{tabular}
\end{table}

\paragraph{Routing cost.} The hierarchical router adds a small amount of local computation, but the systems overhead is limited. Appendix Table~\ref{tab:router_profile} shows that per-layer router time rises only from $1.3$ to $1.6$ ms, while grouped dispatch cuts synchronization overhead by roughly $3\times$ and makes the average forward pass $12\%$ faster than flat routing (42.3 ms vs.\ 48.0 ms). Relative to MoGE, the remaining end-to-end step-time overhead is approximately $5\%$.

\paragraph{Load balancing.} To analyze load balancing in depth, we examine the expert activation patterns across network layers. 
% We recall that our architecture organizes experts into 4 groups (2 experts each), with global expert IDs sequentially numbered from 0 to 7. 

\begin{figure}[!htbp]
    \centering
    \begin{minipage}[b]{0.47\columnwidth}
        \centering
        \caption*{(a) Vanilla GShard-style}\includegraphics[width=\columnwidth]{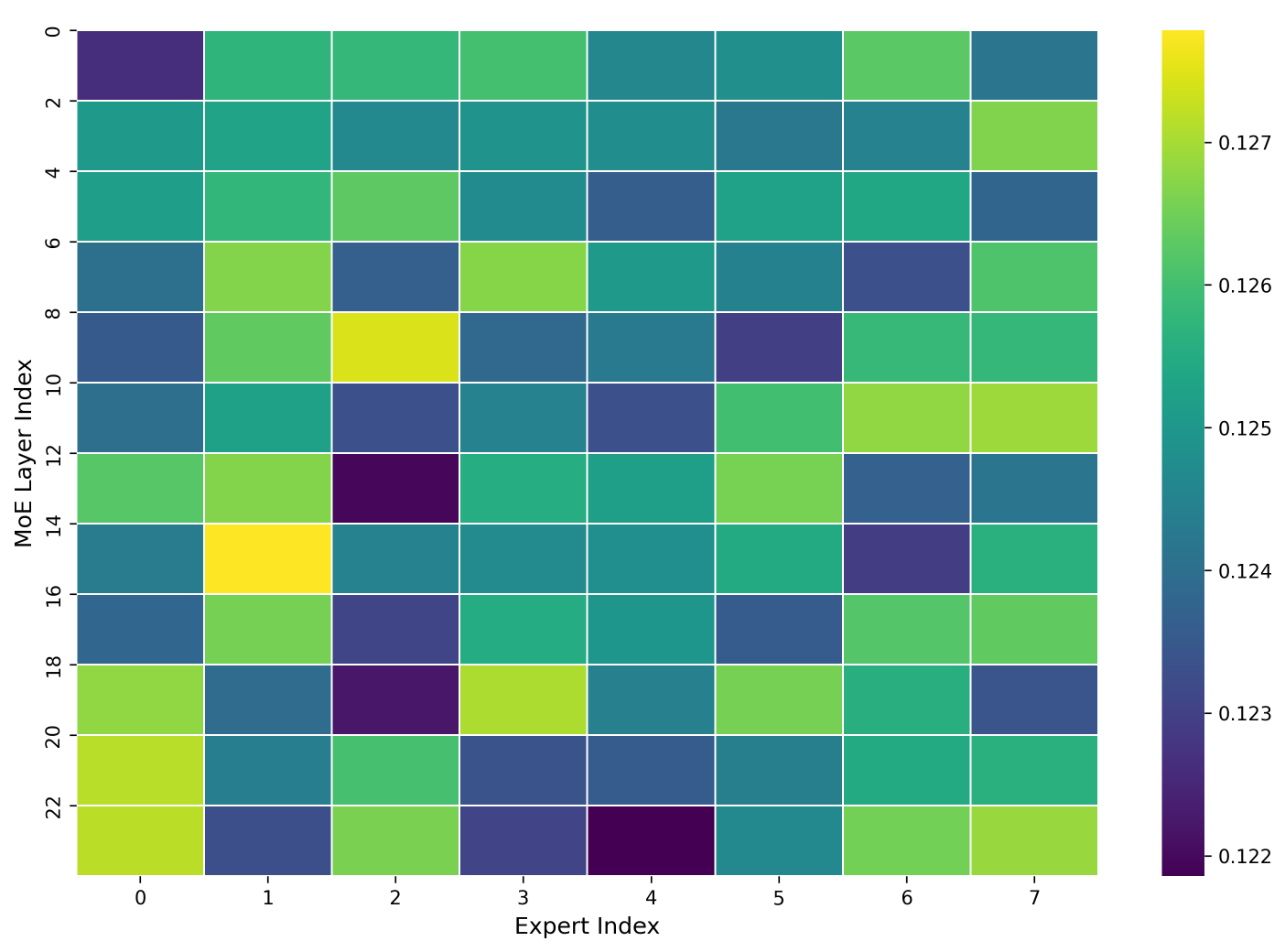}
        
    \end{minipage}
    \hfill
    \begin{minipage}[b]{0.47\columnwidth}
        \centering
        \caption*{(b) Loss-free load balancing}\includegraphics[width=\columnwidth]{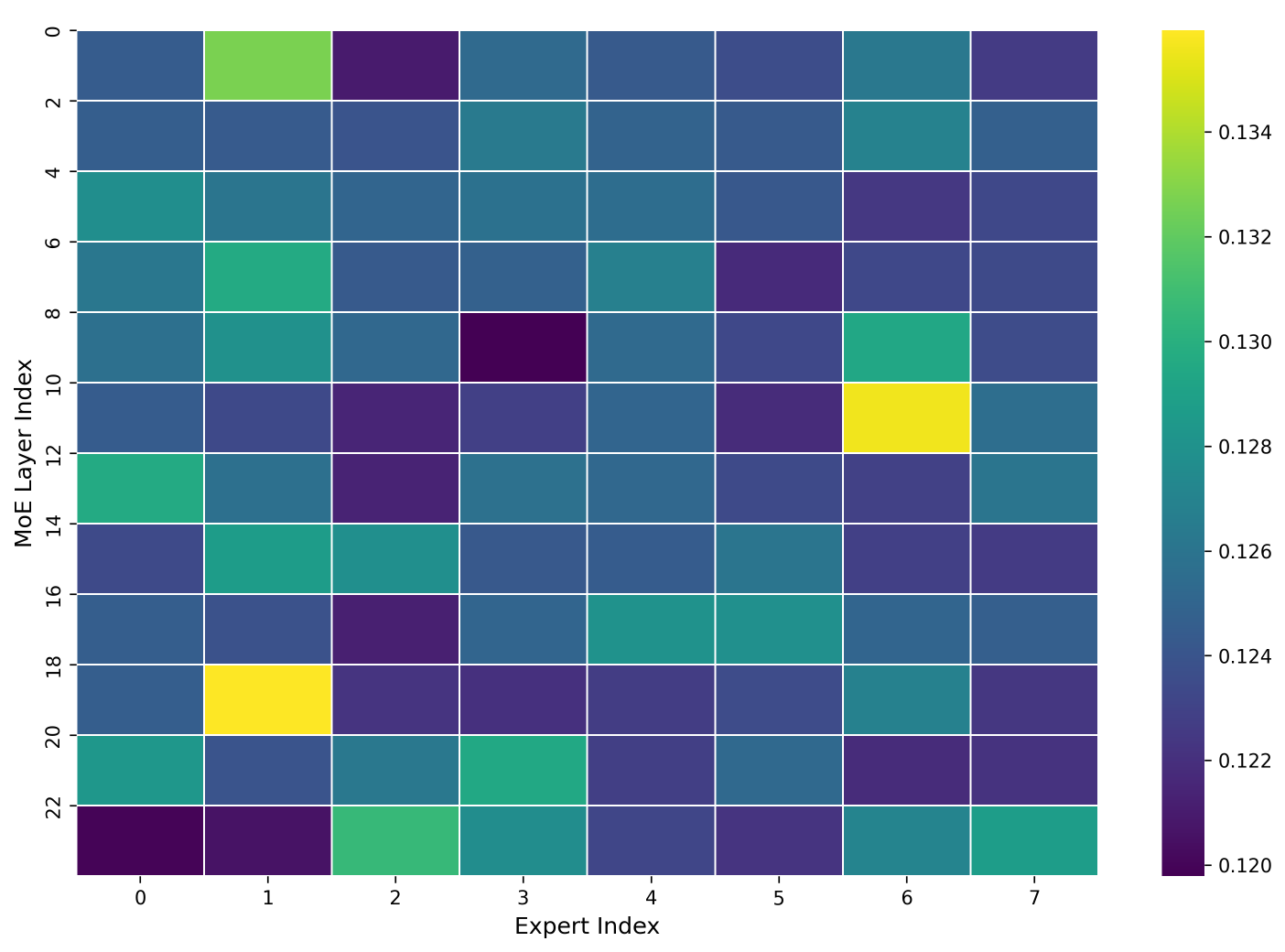}
        
    \end{minipage}

    \begin{minipage}[b]{0.47\columnwidth}
        \centering
        \caption*{(c) MoGE}\includegraphics[width=\columnwidth]{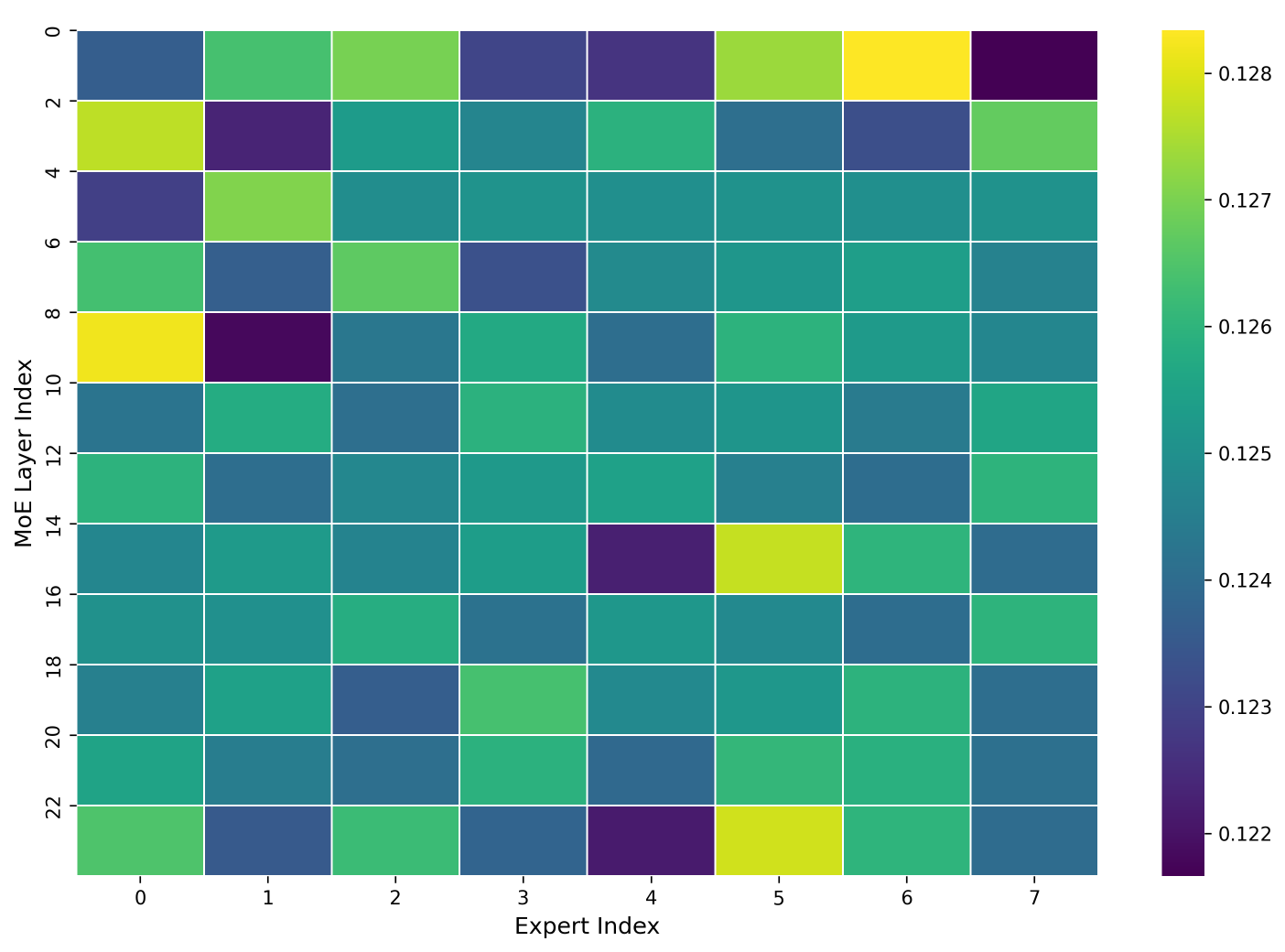}
        
    \end{minipage}
    \hfill
    \begin{minipage}[b]{0.47\columnwidth}
        \centering
        \caption*{(d) \texttt{Hi-MoE}}\includegraphics[width=\columnwidth]{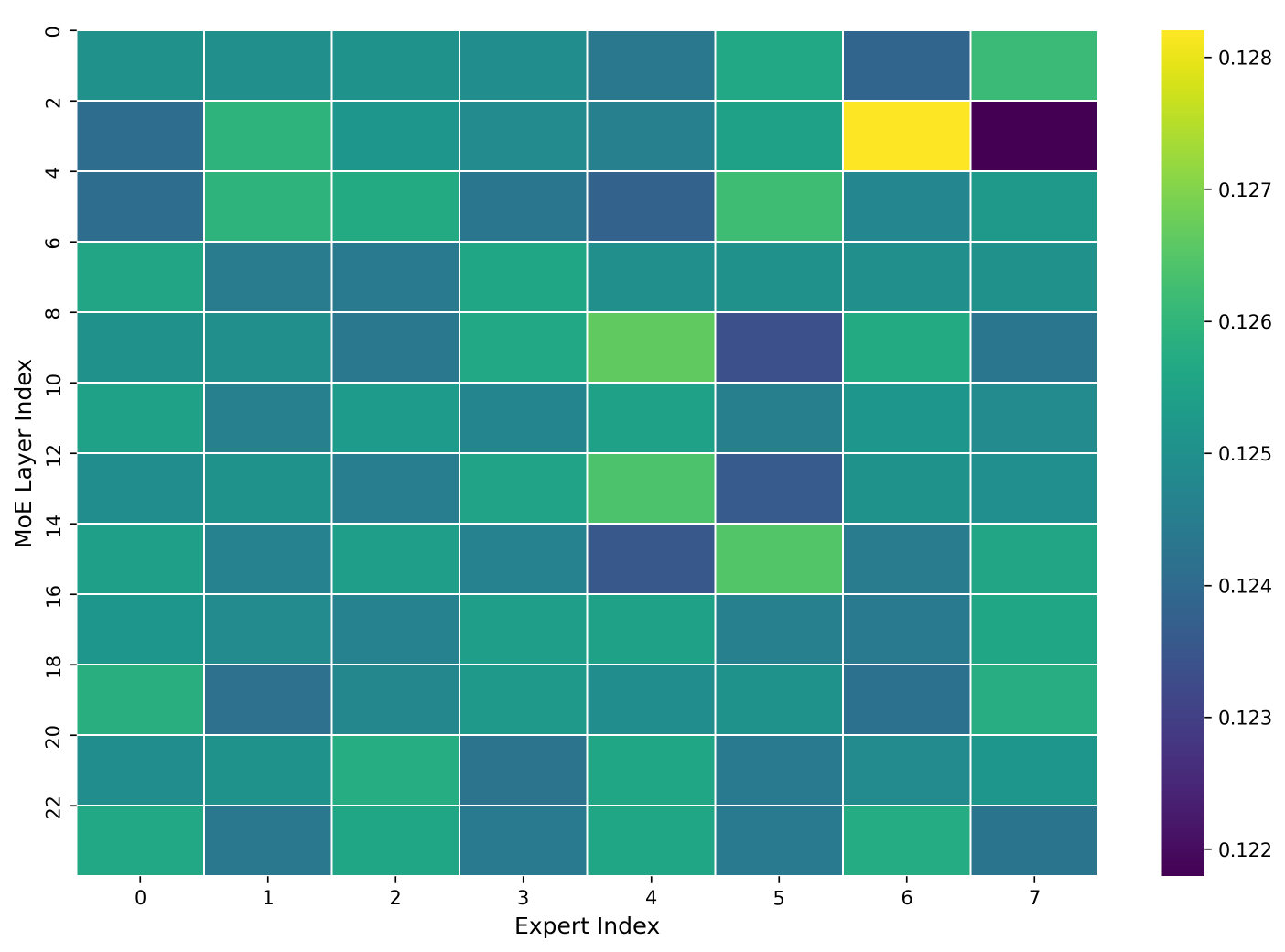}
        
    \end{minipage}
    \caption{Expert activation heatmaps across all 12 MoE layers during nanoGPT-1B training. Subplots show different layers, with experts on the x-axis and training steps on the y-axis.}
    \label{fig:nanogpt_layers}
\end{figure}

Activation maps in Figure~\ref{fig:nanogpt_layers} reveal clear differences in routing. In MoGE (two experts per group), the group-average utilization can look balanced, yet routing is often strongly imbalanced within groups: one expert is consistently preferred while its paired expert is underused. This suggests MoGE mainly stabilizes routing across groups, but does not reliably prevent intra-group imbalance.
Vanilla GShard-style routing may look fine in terms of global averages and can induce less within-device imbalance since it does not enforce group constraints, but it remains noticeably uneven at the device level. Loss-free routing is even less stable, with large under-activated regions and a few sharp spikes.
In contrast, \texttt{Hi-MoE} produces an almost uniform expert-usage pattern.

Figure~\ref{fig:nanogpt_expert_freq} further confirms these observations by showing the final expert activation frequencies in the 5th MoE layer.

\begin{figure}[ht]
    \centering
    \includegraphics[width=0.85\columnwidth]{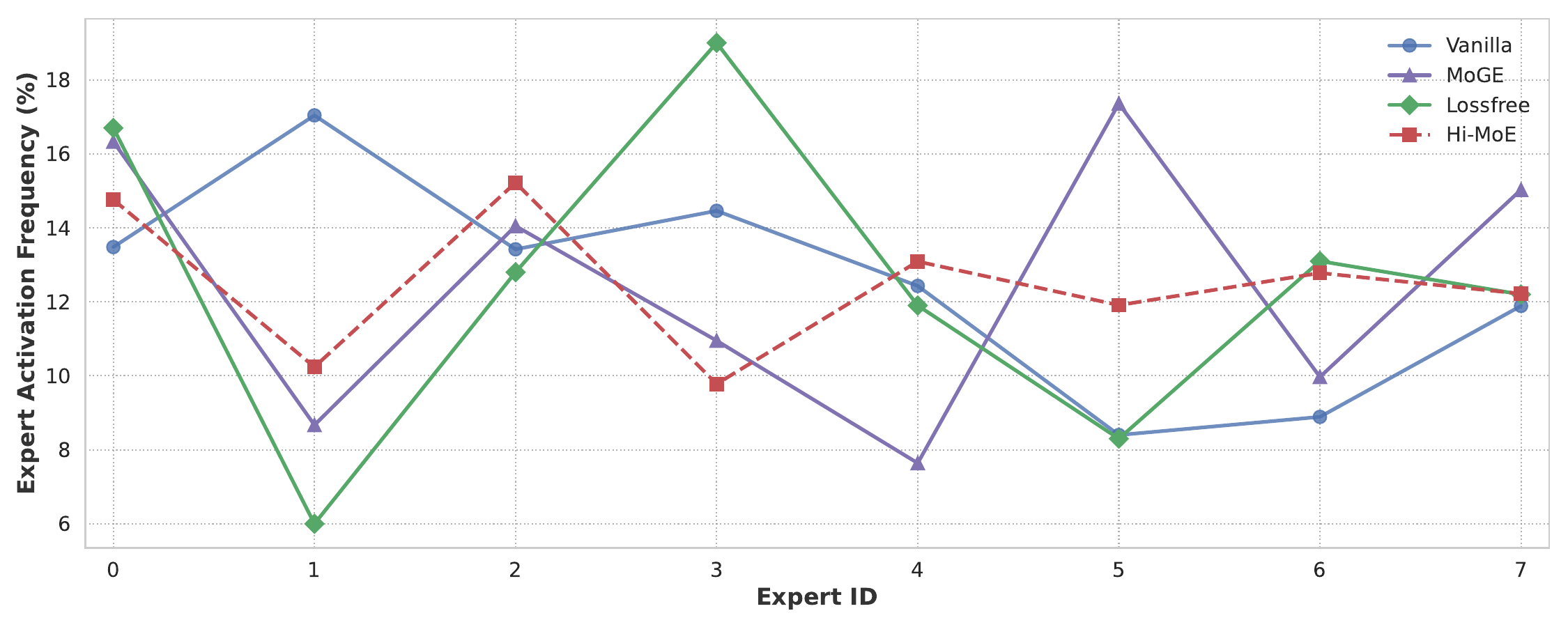}
    \caption{Expert activation frequency distribution in the 5th MoE layer of nanoGPT-1B after training.}
    \label{fig:nanogpt_expert_freq}
\end{figure}
Next, we examine the per-layer group utilization. MoGE and \texttt{Hi-MoE} achieve uniform group load by selecting an equal number of experts from each group. In Figure~\ref{fig:nanogpt_group_layers}, we compare group utilization for Vanilla MoE and the loss-free routing strategy.

\begin{figure}[!htbp]
    \centering
    \begin{minipage}[b]{0.49\columnwidth}
        \centering
        \includegraphics[width=\columnwidth]{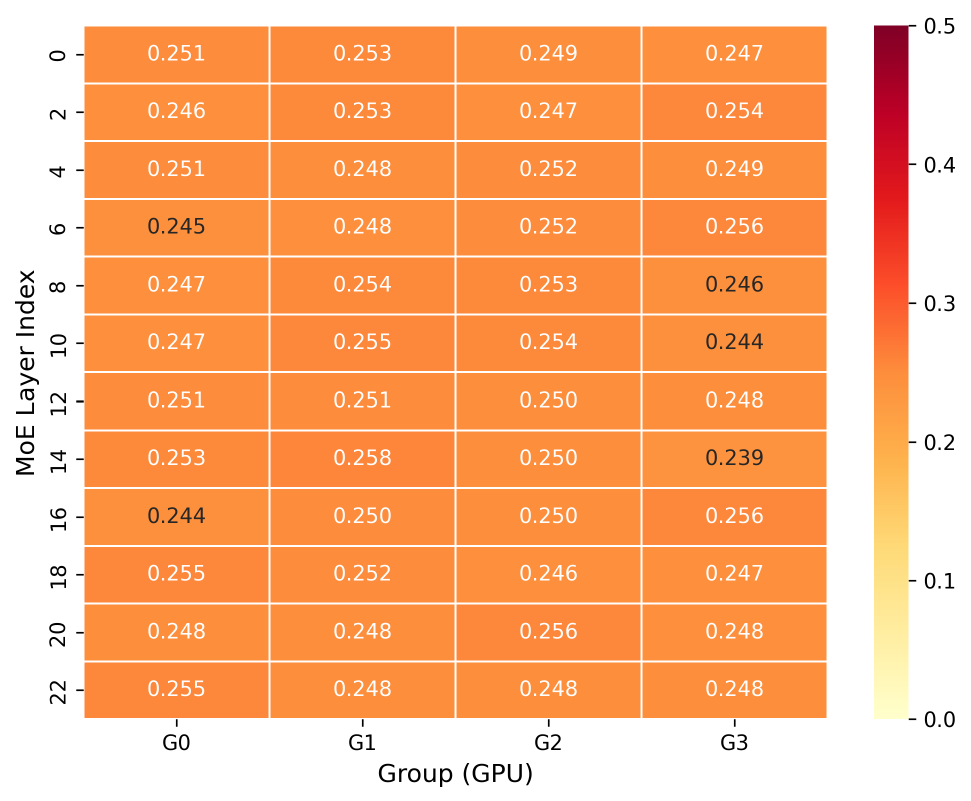}
        \caption*{(a) Vanilla GShard style}
    \end{minipage}
    \hfill
    \begin{minipage}[b]{0.49\columnwidth}
        \centering
        \includegraphics[width=\columnwidth]{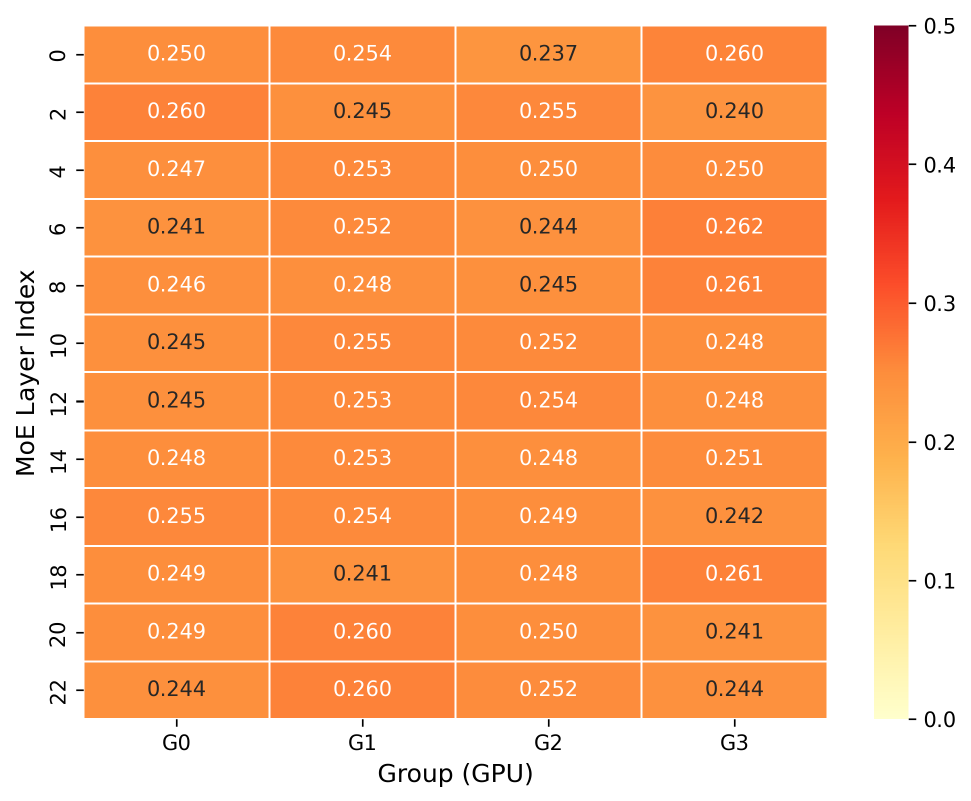}
        \caption*{(b) \texttt{Loss-free load balancing}}
    \end{minipage}
    \caption{Group workload distribution.}
    \label{fig:nanogpt_group_layers}
\end{figure}

Although the aggregate group workload in Figure~\ref{fig:nanogpt_group_layers} appears
nearly uniform, this aggregate view is misleading at the per-token level.
Because Vanilla and loss-free routing do not constrain the number of groups activated
per token, individual tokens tend to concentrate their selected experts on a subset of
groups.
On average, each token activates experts from \emph{only $3.142$ out of $4$ groups}
(GPUs) under Vanilla routing and $3.151$ under loss-free routing.
Thus, despite balanced total workloads, roughly one GPU per
token \emph{receives no computation}, directly reducing device utilization. 
% Nevertheless, MoE modules are typically deployed as components of substantially larger systems. Therefore, we further scale up the model and examine whether the observed benefits persist under more realistic regimes.

\section{Large-Scale Validation: OLMoE-7B vs \texttt{Hi-MoE}-7B}

We perform large-scale pre-training experiments comparing standard OLMoE-7B to \texttt{Hi-MoE}-7B. Both models follow the same backbone configuration: 16 transformer layers with hidden size 2048, 64 experts per MoE layer, and top-$k{=}8$ routing. The \texttt{Hi-MoE} variant additionally structures the 64 experts into 4 hierarchical groups and selects 2 experts per group, preserving the same per-token compute budget. We train both models for 2 epochs on 58.54B tokens from the Dolma mixture, including Algebraic Stack, arXiv, and OpenWebMath.
Due to the computational cost of these large-scale runs, we restrict comparisons to the standard OLMoE-7B baseline. We already benchmarked against a broader set of routing and grouping baselines in earlier sections. Here, our primary goal is to isolate and test the scalability of \texttt{Hi-MoE} under a realistic pre-training regime.

\paragraph{Performance metrics.}
We begin with Table~\ref{tab:olmoe7b_perplexity}, which reports perplexity on 11 validation domains.
\begin{table}[!htbp]
\centering
\caption{Perplexity comparison between OLMoE-7B and \texttt{Hi-MoE}-7B across validation domains.}
\label{tab:olmoe7b_perplexity}
\setlength{\tabcolsep}{4pt}
\begin{tabular}{lccc}
\toprule
\textbf{Domain} & \textbf{OLMoE-7B} & \textbf{\texttt{Hi-MoE}-7B} & \textbf{$\Delta$ (\%)} \\
\midrule
wikitext\_103           & 32.997 & \textbf{31.381} & \textbf{-4.9} \\
c4\_en                  & 39.852 & \textbf{37.345} & \textbf{-6.3} \\
dolma\_books            & 39.704 & \textbf{36.674} & \textbf{-7.6} \\
dolma\_common-crawl     & 42.722 & \textbf{39.872} & \textbf{-6.7} \\
dolma\_pes2o            & 11.685 & \textbf{11.301} & \textbf{-3.3} \\
dolma\_reddit           & 47.023 & \textbf{43.802} & \textbf{-6.8} \\
dolma\_stack            & 3.811  & \textbf{3.765}  & \textbf{-1.2} \\
dolma\_wiki             & 31.831 & \textbf{30.008} & \textbf{-5.7} \\
ice                     & \textbf{23.819} & 23.888 & +0.3 \\
m2d2\_s2orc             & 23.286 & \textbf{21.803} & \textbf{-6.4} \\
pile                    & 13.364 & \textbf{12.978} & \textbf{-2.9} \\
% \midrule
% \textbf{Average}        & 28.190 & \textbf{26.620} & \textbf{-5.6} \\
\bottomrule
\end{tabular}
\end{table}

Beyond perplexity metrics, we evaluate both models on downstream tasks. Table~\ref{tab:olmoe7b_downstream} reports performance on common sense reasoning, question answering, and knowledge benchmarks. All tasks are evaluated with \texttt{lm-evaluation-harness}; unless stated otherwise, results are 0-shot, while MMLU is 5-shot. We report either accuracy (\texttt{acc}) or length-normalized accuracy (\texttt{len\_norm}), following the benchmark default.

\begin{table}[!htbp]
\centering
\caption{Final-checkpoint downstream performance comparison between OLMoE-7B and \texttt{Hi-MoE}-7B.}
\label{tab:olmoe7b_downstream}
\setlength{\tabcolsep}{2pt}
\begin{tabular}{lccc}
\toprule
\textbf{Task} & \textbf{OLMoE-7B} & \textbf{\texttt{Hi-MoE}-7B} & \textbf{$\Delta$ (\%)} \\
\midrule
\multicolumn{4}{l}{\textit{Common Sense Reasoning}} \\
\midrule
COPA (acc)                  & 53.0 & \textbf{57.0} & \textbf{+7.5}  \\
Winogrande (acc)            & 49.1 & \textbf{50.8} & \textbf{+3.5}  \\
CommonsenseQA (len\_norm)   & \textbf{27.4} & 26.6 & -2.9 \\
Social IQA (len\_norm)      & 39.2 & \textbf{39.4} & \textbf{+0.5}  \\
\midrule
\multicolumn{4}{l}{\textit{Question Answering}} \\
\midrule
ARC-Easy (acc)              & 41.4 & \textbf{41.9} & \textbf{+1.2}  \\
ARC-Challenge (len\_norm)   & \textbf{22.7} & 22.1 & -2.6 \\
OpenBookQA (len\_norm)      & 25.2 & \textbf{26.8} & \textbf{+6.3}  \\
% BoolQ (acc)                 & \textbf{57.9} & 56.1 &  \\
% SciQ (acc)                  & \textbf{80.8} & 80.0 &  \\
\midrule
\multicolumn{4}{l}{\textit{General Knowledge}} \\
\midrule
PIQA (len\_norm)            & \textbf{58.1} & 57.5 & -1.0 \\
HellaSwag (len\_norm)       & 30.5 & \textbf{31.3} & \textbf{+2.6}  \\
\midrule
\multicolumn{4}{l}{\textit{MMLU (5-shot MC, avg)}} \\
\midrule
STEM (acc)                         & 21.2 & \textbf{25.9} & \textbf{+22.2} \\
Humanities (acc)      
& 23.6 & \textbf{26.4} & \textbf{+11.9} \\
Social Sciences (acc)              & 24.6 & \textbf{25.8} & \textbf{+4.9}  \\
% Other                       & \textbf{26.2} & 21.7 &  \\
\bottomrule
\end{tabular}
\end{table}
At the final checkpoint, \texttt{Hi-MoE} improves most tasks and all reported MMLU slices, but a few \texttt{len\_norm} tasks remain lower. Because late-checkpoint \texttt{len\_norm} scores can be noisy, we also compare best-checkpoint values on CommonsenseQA and ARC-Challenge, the two most notable final-checkpoint regressions. Table~\ref{tab:olmoe7b_best} shows that CommonsenseQA turns positive under this protocol and the ARC-Challenge gap shrinks substantially, suggesting that the original regressions are largely a checkpoint-selection artifact rather than evidence that hierarchical routing hurts reasoning.

\begin{table}[!htbp]
\centering
\caption{Best-checkpoint comparison on tasks with final-checkpoint regressions (len\_norm metric).}
\label{tab:olmoe7b_best}
\setlength{\tabcolsep}{6pt}
\begin{tabular}{lcc}
\toprule
\textbf{Task} & \textbf{OLMoE-7B best} & \textbf{\texttt{Hi-MoE}-7B best} \\
\midrule
CommonsenseQA  & 27.35 & \textbf{27.44} \\
ARC-Challenge & \textbf{22.74} & 22.44 \\
\bottomrule
\end{tabular}
\end{table}

\paragraph{Load balance.}
In addition to the strong quality gains, an equally important question is routing balance. Indeed, uneven expert utilization is often the main practical bottleneck in MoE training and directly impacts resource efficiency. Table~\ref{tab:olmoe7b_summary} shows a large and operationally meaningful gap: following the standard OLMoE training protocol, \texttt{Hi-MoE-7B} achieves markedly more uniform expert utilization, improving load balance by about $40\%$ relative to OLMoE-7B. This directly translates into better GPU utilization at scale.

\begin{table}[!htbp]
\centering
\caption{Expert balance comparison between OLMoE-7B and \texttt{Hi-MoE}-7B. Lower variation (CV) indicates better balance.}
\label{tab:olmoe7b_summary}
\setlength{\tabcolsep}{8pt}
\begin{tabular}{lcc}
\toprule
\textbf{Metric} & \textbf{OLMoE-7B} & \textbf{\texttt{Hi-MoE}-7B} \\
\midrule
Expert Balance (CV) & 0.257 & \textbf{0.153} \\
\bottomrule
\end{tabular}
\end{table}

\paragraph{Diversity diagnostics.}
Load balance alone does not guarantee non-redundant experts. Appendix Table~\ref{tab:diversity_metrics} therefore reports two complementary diversity diagnostics for OLMoE-7B and \texttt{Hi-MoE}-7B. Parameter-space cosine is near zero for both models and is not discriminative in this regime, whereas representation-level expert similarity is lower for \texttt{Hi-MoE}, indicating better functional diversity.

\section*{Discussion \& Future Work}
In large MoE models, performance is tightly coupled to how routing allocates capacity. \texttt{Hi-MoE} shows that hierarchical coordination (intra- and inter-group) can turn expert scaling into reliable quality gains while keeping utilization well-behaved.

Future work can extend hierarchical routing to multimodal regimes and enable more structured specialization without sacrificing stability. Moreover, the hierarchy provides an interesting direction for mechanistic interpretability: groups can serve as meaningful units for analyzing emergent skills and steering inference by constraining routing to targeted expert subsets. 

\section*{Acknowledgements}
The authors would like to thank Michael Diskin for the valuable discussions throughout the preparation of this manuscript.

% ЗДЕСЬ ОЧЕНЬ СТРАННЫЕ РЕЗУЛЬТАТЫ, БУКВАЛЬНО НАОБОРОТ
% \begin{table*}[t]
% \centering
% \setlength{\tabcolsep}{4pt}
% \begin{tabular}{lcccccccccc}
% \toprule
% \textbf{Method} & \textbf{Test} & \textbf{Test} & \textbf{ARC-E} & \textbf{ARC-E} & \textbf{ARC-C} & \textbf{ARC-C} & \textbf{MMLU} & \textbf{MMLU} & \textbf{AVG} & \textbf{AVG} \\
% & \textbf{PPL} & \textbf{CV} & \textbf{Acc} & \textbf{CV} & \textbf{Acc} & \textbf{CV} & \textbf{Acc} & \textbf{CV} & \textbf{Acc} & \textbf{CV} \\
% \midrule
% Vanilla & $2.985_{\pm 0.024}$ & $0.31_{\pm 0.016}$ & 0.3211 & \textbf{0.2662} & 0.2457 & \textbf{0.2696} & 0.2387 & \textbf{0.2598} & 0.2685 & \textbf{0.2652} \\
% Loss-free & $2.979_{\pm 0.027}$ & $0.37_{\pm 0.017}$ & 0.3255 & 0.3845 & 0.2490 & 0.3698 & 0.2430 & 0.2986 & 0.2725 & 0.3110 \\
% MoGE & $2.977_{\pm 0.025}$ & $0.25_{\pm 0.018}$ & 0.3119 & 0.4036 & \textbf{0.2705} & 0.4032 & 0.2340 & 0.3097 & 0.2721 & 0.3722 \\
% \textbf{\texttt{Hi-MoE}} & $\textbf{2.967}_{\pm 0.019}$ & $\textbf{0.23}_{\pm 0.013}$ & \textbf{0.3367} & 0.3670 & 0.2509 & 0.3646 & \textbf{0.2490} & 0.3274 & \textbf{0.2789} & 0.3530 \\
% \bottomrule
% \end{tabular}
% \caption{Comparison of language modeling performance across nanoGPT architectures. Test metrics are reported as mean$_{\pm \text{std}}$ over 3 runs with different random seeds. Downstream benchmarks (ARC-E, ARC-C, MMLU) are evaluated once per trained model. Best values in each metric are highlighted in bold.}
% \label{tab:combined_results}
% \end{table*}

% \begin{acks}
% To Robert, for the bagels and explaining CMYK and color spaces.
% \end{acks}

\bibliographystyle{ACM-Reference-Format}
\bibliography{refs}

\appendix
% \newpage
\section{Experimental Reproducibility Checklist}
\label{app:repro}
\subsection{\texttt{Hi-MoE} fitting hyperparameters}
To ensure reproducibility and simplify adoption, we summarize the main architectural and optimization hyperparameters used to fit \texttt{Hi-MoE} across our experimental setups. Table~\ref{tab:hyperparameters} reports the configurations alongside model-specific settings.
\begin{table}[!htbp]
\centering
\caption{\texttt{Hi-MoE} hyperparameters for different setups.}
\label{tab:hyperparameters}
\small
\setlength{\tabcolsep}{3pt}
\begin{tabular}{llll}
\toprule
\textbf{Hyperparameter} & \textbf{nanoGPT} & \textbf{Swin-MoE} & \textbf{\texttt{OLMoE}-7B} \\
\midrule
Hidden dimension ($d$) & 1,024 & 96 & 2,048 \\
Attention heads & 16 & [3, 6, 12, 24] & 16 \\
Layers & 24 & [2, 2, 18, 2] & 16 \\
Context / image size & 1,024 tokens & $192 \times 192$ & 4,096 tokens \\
\midrule
Total experts ($N$) & 8 & 64 & 64 \\
Expert groups ($M$) & 4 & 4 & 4 \\
Top-$k$ per group & 1 & 1 & 2 \\
Active experts per token & 4 & 4 & 8 \\
Expert FFN dimension & $4 \times 1{,}024 = 4{,}096$ & $4 \times d_{\text{stage}}$ & 1,024 \\
$\alpha$ ($\mathcal{L}_{\text{load}}$ coefficient) & $0.01^\dagger$ & $0.01^\dagger$ & $0.0001^\dagger$ \\
\midrule
$T$ & $1.0^*$ & $1.0^*$ & 1.0 \\
$\lambda_{\text{intra}}$ & $0.1^*$ & $0.1^*$ & 0.1 \\
$\lambda_{\text{inter}}$ & $0.05^*$ & $0.05^*$ & 0.05 \\
$\tau$ & $0.01^*$ & $0.01^*$ & 0.01 \\
$\beta$ & $0.9^*$ & $0.9^*$ & 0.9 \\
$\beta_1$ (AdamW) & $0.9^\dagger$ & $0.9^\dagger$ & $0.9^\dagger$ \\
$\beta_2$ (AdamW) & $0.95^\dagger$ & $0.999^\dagger$ & $0.95^\dagger$ \\
\bottomrule
\end{tabular}

\begin{flushleft}
\scriptsize $^*$Tuned over: $T \in \{0.5, 0.75, 1.0, 1.25, 1.5\}$, $\lambda \in \{10^{-6}, 10^{-5}, 10^{-4}, 10^{-3}, 10^{-2}, 0.05, 0.1\}$, $\tau \in \{1, 0.1, 0.01, 0.001, 0.0001\}$, $\beta \in \{0.1, 0.3, 0.5, 0.7, 0.8, 0.9, 0.95, 0.99, 0.999\}$. For \texttt{Hi-MoE}-7B, these parameters were not tuned due to model scale and computational constraints.

\scriptsize $^\dagger$Taken from original codebase configurations.
\end{flushleft}
\end{table}

\paragraph{GPU utilization.}
For system-level efficiency, one can adopt the same distributed training pipeline as Pangu MoE~\citep{moge}, combining tensor, expert, pipeline (including virtual pipeline), and context parallelism. While \citep{moge} also highlights engineering-driven throughput gains under such setups, our work intentionally focuses on the methodological side -- improving routing behavior via hierarchical optimization. Importantly, \texttt{Hi-MoE} is fully compatible with these parallelism pipelines and, by enforcing more uniform expert utilization, can translate into higher effective GPU utilization when deployed in the same system regime.

\subsection{Router profiling}
\label{app:repro:profiling}
Table~\ref{tab:router_profile} profiles nanoGPT-1B routing under identical conditions ($4\times$ H100, batch size 12). \texttt{Hi-MoE} increases per-layer router time because grouped Top-$K$ and bias correction add a small amount of local compute, but grouped dispatch reduces synchronization overhead enough to make the average forward pass faster than flat routing. Relative to MoGE, the remaining end-to-end step-time overhead is approximately $5\%$.

\begin{table}[!htbp]
\centering
\caption{nanoGPT-1B routing profile under identical conditions ($4\times$ H100, batch size 12).}
\label{tab:router_profile}
\setlength{\tabcolsep}{7pt}
\begin{tabular}{lcc}
\toprule
\textbf{Metric} & \textbf{Flat MoE} & \textbf{\texttt{Hi-MoE}} \\
\midrule
Per-layer router time (ms) & 1.3 & 1.6 \\
\texttt{cudaDeviceSynchronize} (relative) & $1.0\times$ & $0.33\times$ \\
Average forward pass (ms) & 48.0 & \textbf{42.3} \\
\bottomrule
\end{tabular}
\vspace{-4mm}
\end{table}

\subsection{Additional diversity diagnostics}
\label{app:repro:diversity}
To complement CV, we report both a parameter-space metric and a representation-level metric for the 7B models. Table~\ref{tab:diversity_metrics} shows that parameter cosine is nearly zero for both methods and therefore weakly informative, whereas mean expert similarity decreases under \texttt{Hi-MoE}, indicating improved functional diversity.

\begin{table}[!htbp]
\centering
\caption{Additional diversity diagnostics on OLMoE-7B and \texttt{Hi-MoE}-7B.}
\label{tab:diversity_metrics}
\setlength{\tabcolsep}{8pt}
\begin{tabular}{lcc}
\toprule
\textbf{Metric} & \textbf{OLMoE-7B} & \textbf{\texttt{Hi-MoE}-7B} \\
\midrule
Parameter cosine & $0.012 \pm 0.004$ & $0.013 \pm 0.002$ \\
Mean expert similarity & $0.609 \pm 0.158$ & $\mathbf{0.584 \pm 0.105}$ \\
\bottomrule
\end{tabular}
\vspace{-4mm}
\end{table}

%%%%%%%%%%%%%%%%%%%%%%%%%%%%%%%%%%%%%%%%%%%%%%%%%%%%%%%%%%%%%%%%%%%%%%%%%%%%%%
% Theory
%%%%%%%%%%%%%%%%%%%%%%%%%%%%%%%%%%%%%%%%%%%%%%%%%%%%%%%%%%%%%%%%%%%%%%%%%%%%%%

% \subsection*{Notation}
% The table below summarizes the notation used throughout the paper.

\begin{table}[!htbp]
\caption{Notation.}\label{tab:notation}
\centering
\small
\setlength{\tabcolsep}{6pt}
\renewcommand{\arraystretch}{1.15}
\begin{tabular}{p{0.22\columnwidth}p{0.72\columnwidth}}
\toprule
\textbf{Symbol} & \textbf{Meaning}\\
\midrule
$C(\mathbf{x})$ & total (soft) cross-expert gradient coupling on token $\mathbf{x}$ \\
$d$ & token/hidden dimension (so $\mathbf{x}\in\mathbb{R}^d$) \\
$E\sim\pi(\mathbf{X})$ & stochastic expert identity under router $\pi$ \\
$f_i:\mathbb{R}^d\to\mathbb{R}^d$ & expert network $i$ \\
$g:\mathbb{R}^d\to\mathbb{R}^N$ & router / gating network \\
$g_i(\mathbf{x})$ & expert-local gradient on token $\mathbf{x}$ (w.r.t. $\theta_i$) \\
$G$ & uniform bound with $\|g_i(\mathbf{x})\|_2\le G$ \\
$H_2(E)$ & collision entropy of $E$ \\
% $H_2(E\mid \mathbf{X})$ & collision conditional entropy of expert given token \\
$I_2(\mathbf{X};E)$ & collision mutual information between token and expert \\
$K$ & Top-$K$ sparsity (number of selected experts) \\
$K_m$ & Top-$K$ within group $m$ \\
$L$ & (soft) group load vector for a batch \\
$L_g$ & load assigned to group $g$ \\
$M$ & number of expert groups \\
$N$ & number of experts \\
$p(e\mid\mathbf{x})$ & router distribution over experts (generic notation) \\
$p(g\mid\mathbf{x})$ & induced group mass for token $\mathbf{x}$ \\
$p(i)$ & marginal usage probability of expert $i$ \\
$r(\mathbf{x})\in\Delta^{M-1}$ & normalized group assignment induced by $\widetilde{\bm{\pi}}(\mathbf{x})$ \\
$\bar r$ & mean group-load distribution $\mathbb{E}[r(\mathbf{X})]$ \\
$S_{\max}$ & maximum group size $\max_{g\in[M]}|\mathcal{G}_g|$ \\
$T$ & softmax temperature in bias-corrected router \\
$\mathcal{B}=\{\mathbf{x}_b\}_{b=1}^B$ & batch of $B$ token representations \\
$\mathcal{C}_{\text{ov}}$ & proxy cost for routing overlap \\
$\mathcal{C}_{\text{sys}}$ & proxy cost for group/GPU imbalance \\
$\Delta^{n-1}$ & probability simplex in $\mathbb{R}^n$ \\
$\mathbb{I}\{\cdot\}$ & indicator function \\
$\mathcal{G}_g$ & index set of experts in group $g$ \\
$\mathcal{L}$ & total training objective\\
$\mathcal{L}_{\text{load}}$ & standard auxiliary load-balancing loss \\
$\mathcal{L}_{\text{task}}$ & task loss \\
$\mathcal{R}_{\text{inter}}$ & inter-group regularizer ($\propto\|\widetilde{\bm{\pi}}\|_2^2$) \\
$\mathcal{R}_{\text{intra}}$ & intra-group regularizer ($\propto-\|\bm{\pi}\|_2^2$) \\
$\bm{\pi}(\mathbf{x})\in\Delta^{N-1}$ & router probabilities over experts for token $\mathbf{x}$ \\
$\pi_i(\mathbf{x})$ & $i$-th component of $\bm{\pi}(\mathbf{x})$ \\
$\widetilde{\bm{\pi}}(\mathbf{x})$ & post-Top-$K$ routing weights  \\
$\mathbf{x},\mathbf{X}$ & token representation (deterministic / random variable) \\
$\overline{\mathbf{g}}$ & EMA of router logits \\
$\theta_i\in\mathbb{R}^P$ & parameters of expert $i$ \\
$P$ & number of parameters per expert \\
$\lambda_{\text{inter}}$ & weight of $\mathcal{R}_{\text{inter}}$ \\
$\lambda_{\text{intra}}$ & weight of $\mathcal{R}_{\text{intra}}$ \\
$\mu$ & mean of $\{L_g\}_{g=1}^M$ \\
$\mu_{\text{load}}$ & mean of token counts per expert \\
$\sigma^2$ & variance of $\{L_g\}_{g=1}^M$ \\
$\sigma_{\text{load}}$ & std of token counts per expert \\
$\tau$ & bias-correction strength in router logits \\
$\widehat L\in\Delta^{M-1}$ & normalized load distribution $\widehat L_g=L_g/\sum_h L_h$ \\
$\mathrm{CV}(L)$ & coefficient of variation of $L$: $\sigma/\mu$ \\
$\varepsilon_{\text{ov}}$ & overlap constraint level \\
$\varepsilon_{\text{sys}}$ & system-imbalance constraint level \\
\bottomrule
\end{tabular}
\end{table}
\normalsize

% Table \ref{tab:notation} summarizes the notation used throughout the paper.
% \newpage

\subsection{Hardware}
\label{app:repro:hardware}

Experiments were run on nodes with $8\times$ NVIDIA H200 GPUs and $8\times$ NVIDIA H100 GPUs. As an end-to-end reference point, pre-training Swin Transformer-1B takes 24 GPU-hours, nanoGPT-1B takes 372 GPU-hours, and \texttt{Hi-MoE}-7B takes approximately 632 GPU-hours.

\section{Supplementary analysis}
\label{sec:theory}
In this appendix, we provide proofs for all theoretical statements from Section~\ref{sec:main:theory}.

Recall that $M$ denotes the number of expert groups (e.g., GPU-aligned shards). Consider any nonnegative group load vector
$L=(L_1,\dots,L_M)\in\mathbb{R}_+^M$ computed on a batch (hard counts or soft mass, both are admissible here).
Let $\mu := \frac{1}{M}\sum_{g=1}^M L_g$ and $\sigma^2 := \frac{1}{M}\sum_{g=1}^M (L_g-\mu)^2$, and define
$\mathrm{CV}(L):=\sigma/\mu$ as in \eqref{eq:load}.
Define the \emph{normalized} load distribution $\widehat L \in \Delta^{M-1}$ by
$\widehat L_g := L_g / \sum_{h=1}^M L_h$.

\begin{proposition}[CV--$\ell_2$ equivalence]
\label{thm:cv-l2}
For any $L\in\mathbb{R}_+^M$ with $\sum_g L_g>0$,
\[
\mathrm{CV}(L)^2 \;=\; M\bigl\|\widehat L\bigr\|_2^2 \;-\; 1.
\]
In particular, minimizing $\mathrm{CV}(L)$ is equivalent to minimizing $\|\widehat L\|_2^2$.
\end{proposition}

\begin{proof}
Let $S:=\sum_{g=1}^M L_g$, so $\mu=S/M$ and $L_g-\mu=S(\widehat L_g-1/M)$.
Then
\[
\sigma^2
= \frac{1}{M}\sum_{g=1}^M S^2\Bigl(\widehat L_g-\frac{1}{M}\Bigr)^2
= \frac{S^2}{M}\sum_{g=1}^M \Bigl(\widehat L_g^2 - \frac{2}{M}\widehat L_g + \frac{1}{M^2}\Bigr).
\]
Since $\sum_g \widehat L_g = 1$, we obtain
$\sum_g (\widehat L_g-\frac{1}{M})^2 = \sum_g \widehat L_g^2 - \frac{1}{M} = \|\widehat L\|_2^2 - \frac{1}{M}$.
\begin{eqnarray*}
\sigma^2
&=& \frac{S^2}{M}\Bigl(\|\widehat L\|_2^2 - \frac{1}{M}\Bigr),\\
\mathrm{CV}(L)^2
&=& M\|\widehat L\|_2^2 - 1.
\end{eqnarray*}
\end{proof}

Proposition~\ref{thm:cv-l2} shows that \emph{balance is exactly an $\ell_2$ concentration objective}.
Hence any surrogate that provably controls an $\ell_2$-concentration of group loads is directly tied to CV.
%-------------------------------------------------------------------------------
\subsection{Inter-group regularizer: additional proofs}
\label{sec:theory-inter}

We now connect the inter-group regularizer $\mathcal{R}_{\text{inter}}$ in \eqref{eq:r_inter}
to group imbalance.
For a token representation $\mathbf{x}$, let $\widetilde{\bm{\pi}}(\mathbf{x})\in\mathbb{R}_+^N$ be the
post-selection (sparsified) weights defined in \eqref{eq:topk_pi}.
Define the induced \emph{group-mass} vector
$r(\mathbf{x})\in\mathbb{R}_+^M$ by
\[
r_g(\mathbf{x}) \;:=\; \sum_{e\in\mathcal{G}_g}\widetilde{\pi}_e(\mathbf{x}),\qquad g=1,\dots,M.
\]
When $\widetilde{\bm{\pi}}(\mathbf{x})$ is normalized to sum to one on the selected support (the common Top-$K$-softmax variant),
$r(\mathbf{x})$ is a distribution over groups and $\sum_g r_g(\mathbf{x})=1$.
(If one uses the softmax-then-mask variant in \eqref{eq:topk_pi}, one may equivalently work with the normalized vector
$r(\mathbf{x})/\sum_h r_h(\mathbf{x})$; the inequalities below then hold verbatim for the normalized version.) $S_{\max}:=\max_{g}|\mathcal{G}_g|$ is the largest group size.

\begin{lemma}[Group-sum $\ell_2$ bound (Cauchy--Schwarz)]
\label{lem:group-sum-bound}
For any nonnegative $w\in\mathbb{R}_+^N$ and group partition $\{\mathcal{G}_g\}_{g=1}^M$, let
$m_g := \sum_{e\in\mathcal{G}_g} w_e$. Then
\[
\|m\|_2^2 \;\le\; S_{\max}\,\|w\|_2^2.
\]
\end{lemma}

\begin{proof}
Fix a group $g$. By Cauchy--Schwarz,
\[
m_g^2 = \Bigl(\sum_{e\in\mathcal{G}_g} w_e\Bigr)^2
\le |\mathcal{G}_g|\sum_{e\in\mathcal{G}_g} w_e^2
\le S_{\max}\sum_{e\in\mathcal{G}_g} w_e^2.
\]
Summing over $g$ yields
$\sum_g m_g^2 \le S_{\max}\sum_{g}\sum_{e\in\mathcal{G}_g} w_e^2 = S_{\max}\|w\|_2^2$.
\end{proof}

\begin{theorem}[Theorem \ref{thm:main:inter-controls-cv}]
\label{thm:inter-controls-cv}
Assume $r(\mathbf{x})$ is normalized so that $\sum_{g=1}^M r_g(\mathbf{x})=1$ for all tokens.
Let $\mathbf{X}$ be a random token representation.
Define the \emph{mean group load distribution} $\bar r := \mathbb{E}[r(\mathbf{X})]\in\Delta^{M-1}$.
Then
\[
\|\bar r\|_2^2
\;\le\;
\mathbb{E}\bigl[\|r(\mathbf{X})\|_2^2\bigr]
\;\le\;
S_{\max}\,\mathbb{E}\bigl[\|\widetilde{\bm{\pi}}(\mathbf{X})\|_2^2\bigr].
\]
Consequently, for any batch whose (soft) group-load vector is proportional to $\bar r$, its group-level coefficient of
variation satisfies
\[
\mathrm{CV}_{\text{group}}^2 \;+\; 1
\;=\;
M\|\bar r\|_2^2
\;\le\;
M S_{\max}\,\mathbb{E}\bigl[\|\widetilde{\bm{\pi}}(\mathbf{X})\|_2^2\bigr].
\]
\end{theorem}

\begin{proof}
The first inequality is Jensen's inequality for the convex function $v\mapsto \|v\|_2^2$:
$\|\mathbb{E}[r(\mathbf{X})]\|_2^2 \le \mathbb{E}[\|r(\mathbf{X})\|_2^2]$.
The second inequality is Lemma~\ref{lem:group-sum-bound} applied pointwise with
$w=\widetilde{\bm{\pi}}(\mathbf{X})$ and $m=r(\mathbf{X})$.
Finally, by Proposition \ref{thm:cv-l2} and the fact that $\bar r$ is a distribution over $M$ groups,
$\mathrm{CV}_{\text{group}}^2 = M\|\bar r\|_2^2 - 1$.
\end{proof}
%-------------------------------------------------------------------------------
\subsection{Intra-group anti-regularizer: additional proofs}
\label{sec:theory-intra}

If balancing pressure drives routing toward token-independent (or nearly token-independent) assignments,
experts tend to be trained on essentially the same signal, yielding redundant solutions.
% The following symmetry statement makes this precise in a minimal setting.

% \begin{proposition}[Uniform routing preserves expert symmetry]
% \label{prop:symmetry}
% Consider an MoE layer with $N$ experts $\{f_i(\cdot;\theta_i)\}_{i=1}^N$ of identical architecture.
% Assume routing is uniform for every token: $\pi_i(\mathbf{x})\equiv \frac{1}{N}$ for all $i,\mathbf{x}$.
% Assume full-batch gradient descent on any differentiable loss $\mathcal{L}(\theta_1,\dots,\theta_N)$ that depends on experts
% only through the symmetric mixture output $\frac{1}{N}\sum_{i=1}^N f_i(\mathbf{x};\theta_i)$.
% If $\theta_1(0)=\cdots=\theta_N(0)$, then $\theta_1(t)=\cdots=\theta_N(t)$ for all iterations $t$.
% \end{proposition}

% \begin{proof}
% By symmetry of the mixture, for any iteration $t$ with $\theta_1(t)=\cdots=\theta_N(t)$, the partial derivatives
% $\nabla_{\theta_i}\mathcal{L}$ are identical across $i$ (each expert appears in the mixture with the same coefficient).
% Thus the gradient descent updates are identical, preserving $\theta_1(t+1)=\cdots=\theta_N(t+1)$.
% The claim follows by induction from the initialization.
% \end{proof}

In contrast, \texttt{Hi-MoE} adds the intra term
$\mathcal{R}_{\text{intra}}=-\lambda_{\text{intra}}\|\bm{\pi}(\mathbf{x})\|_2^2$ in \eqref{eq:r_intra}.
Note that on the probability simplex, increasing $\|\bm{\pi}\|_2^2$ makes routing \emph{more decisive}
(peaked), which is exactly what breaks symmetry and prevents experts from receiving identical mixtures.
Thus, for any probability vector $\pi\in\Delta^{N-1}$,
\begin{equation}\label{lem:overlap}
\sum_{i\neq j}\pi_i\pi_j \;=\; 1-\|\pi\|_2^2.
\end{equation}
Equivalently, if $E_1,E_2\stackrel{\text{i.i.d.}}{\sim}\pi$, then
$\mathbb{P}(E_1\neq E_2)=1-\|\pi\|_2^2$.

Equation~\eqref{lem:overlap} naturally gives us that maximizing $\|\bm{\pi}(\mathbf{x})\|_2^2$ is exactly
minimizing the soft co-activation overlap $\sum_{i\neq j}\pi_i(\mathbf{x})\pi_j(\mathbf{x})$.
Thus, $\mathcal{R}_{\text{intra}}$ is a principled anti-overlap regularizer:
it explicitly reduces how much two different experts share the same token in expectation.
This is conceptually aligned with the long line of work on encouraging diversity via anti-correlation in ensembles
(e.g., negative correlation learning) and with modern anti-regularization ideas that enlarge the space of hypotheses
sampled by an ensemble (e.g., DARE)~\citep{liu1999ensemble, antiregularization}.
%-------------------------------------------------------------------------------
% \subsection{Anti-overlap increases expert diversity by decoupling gradients}
% \label{sec:theory-gradients}

To connect overlap to what changes in training dynamics, we analyze gradient coupling.
Let each expert have parameters $\theta_i\in\mathbb{R}^P$ (same dimension $P$ for all experts).
For a token $\mathbf{x}$ and its associated loss contribution, denote the expert-local gradient by
$g_i(\mathbf{x}) := \nabla_{\theta_i}\ell_i(\mathbf{x};\theta_i)\in\mathbb{R}^P$.
The MoE weighting yields an effective update direction proportional to
$\pi_i(\mathbf{x})\,g_i(\mathbf{x})$.

\begin{theorem}[Theorem \ref{thm:main:grad-coupling}]
\label{thm:grad-coupling}
Assume $\|g_i(\mathbf{x})\|_2 \le G$ for all $i$ and all $\mathbf{x}$ (clipping scenario).
Define the total (soft) cross-expert coupling on a token by
\[
C(\mathbf{x})
\;:=\;
\sum_{i\neq j}\bigl|\langle \pi_i(\mathbf{x})g_i(\mathbf{x}),\,\pi_j(\mathbf{x})g_j(\mathbf{x})\rangle\bigr|.
\]
Then
\[
C(\mathbf{x})
\;\le\;
G^2\sum_{i\neq j}\pi_i(\mathbf{x})\pi_j(\mathbf{x})
\;=\;
G^2\bigl(1-\|\bm{\pi}(\mathbf{x})\|_2^2\bigr).
\]
Therefore, maximizing $\|\bm{\pi}(\mathbf{x})\|_2^2$ (i.e., applying $\mathcal{R}_{\text{intra}}$)
provably decreases an upper bound on how strongly different experts are trained on the \emph{same} token update.
\end{theorem}

\begin{proof}
By Cauchy--Schwarz and the gradient bound,
\[
\bigl|\langle \pi_i g_i, \pi_j g_j\rangle\bigr|
\le \pi_i\pi_j\|g_i\|_2\|g_j\|_2
\le G^2\pi_i\pi_j.
\]
Summing over $i\neq j$ gives
$C(\mathbf{x})\le G^2\sum_{i\neq j}\pi_i\pi_j$.
Apply Equation~\eqref{lem:overlap} to obtain the final expression.
\end{proof}
%-------------------------------------------------------------------------------
% \subsection{Specialization under balance: maximizing $\|\pi\|_2^2$ increases routing information}
% \label{sec:theory-mi}

To formalize coverage of experts, we use a standard specialization proxy:
how informative the expert identity is about the input token.
Consider stochastic routing where $E\sim \pi(\mathbf{X})$ given a random token $\mathbf{X}$.
A quantity tightly connected to $\|\pi\|_2^2$ is the collision conditional entropy which decreases when routing becomes more decisive:
\[
H_2(E\mid \mathbf{X})
\;:=\;
-\log \mathbb{E}_{\mathbf{X}}\Bigl[\sum_{i=1}^N \pi_i(\mathbf{X})^2\Bigr]
\;=\;
-\log \mathbb{E}_{\mathbf{X}}\bigl[\|\pi(\mathbf{X})\|_2^2\bigr].
\]

\begin{theorem}[Theorem \ref{thm:main:renyi-mi}]
\label{thm:renyi-mi}
Let $p(i):=\mathbb{E}_{\mathbf{X}}[\pi_i(\mathbf{X})]$ be the marginal expert usage probability.
Define the collision mutual information
\[
I_2(\mathbf{X};E)
\;:=\;
H_2(E) - H_2(E\mid \mathbf{X}),
\qquad
H_2(E):=-\log\sum_{i=1}^N p(i)^2.
\]
If expert usage is perfectly balanced, i.e.\ $p(i)=1/N$ for all $i$, then
\[
I_2(\mathbf{X};E)
\;=\;
\log\Bigl(N\,\mathbb{E}_{\mathbf{X}}[\|\pi(\mathbf{X})\|_2^2]\Bigr),
\]
which is strictly increasing in $\mathbb{E}[\|\pi(\mathbf{X})\|_2^2]$.
\end{theorem}

\begin{proof}
By definition,
\[
I_2(\mathbf{X};E)
= -\log\sum_i p(i)^2 + \log \mathbb{E}_{\mathbf{X}}\Bigl[\sum_i \pi_i(\mathbf{X})^2\Bigr]
= \log\!\Bigl(\frac{\mathbb{E}[\|\pi(\mathbf{X})\|_2^2]}{\sum_i p(i)^2}\Bigr).
\]
If $p(i)=1/N$, then $\sum_i p(i)^2 = N\cdot (1/N)^2 = 1/N$, hence
$I_2(\mathbf{X};E)=\log(N\,\mathbb{E}[\|\pi(\mathbf{X})\|_2^2])$.
\end{proof}
\end{document}